\documentclass[12pt]{article}

\usepackage{geometry}

\usepackage{ijcai16}

    \usepackage{times}
\usepackage{url}
\makeatletter
\def\url@leostyle{%
  \@ifundefined{selectfont}{\def\UrlFont{\sf}}{\def\UrlFont{\small\ttfamily}}}
\makeatother
\usepackage{amssymb}
\newtheorem{theorem}{Theorem}
\newtheorem{definition}{Definition}

\newtheorem{proposition}{Proposition}

\newtheorem{example}{Example}

\newtheorem{open question}{Open Question}

\newcommand{\bel}[1]{[#1]}
\newcommand{\contract}{\div} %alternative for the contraction symbol
\newcommand{\combi}{\preceq_{1 \oplus 2}}
\newcommand{\mods}[1]{[\![#1]\!]}
\newcommand{\Cn}[1]{\textrm{Cn}(#1)}
\newcommand{\STQ}{\oplus_{\mathrm{STQ}}}

\newenvironment{proof}{\par\vspace*{0.5ex}\noindent\rm {\bf Proof:} }
{
}

\title{Extending the Harper Identity to Iterated Belief Change}
\author{Richard Booth \\
Cardiff University \\
Cardiff, UK \\
boothr2@cardiff.ac.uk 
\And
Jake Chandler \\
La Trobe University \\
Melbourne, Australia \\
jacob.chandler@latrobe.edu.au}

\begin{document}

%\maketitle

 \fontsize{12}{14}\selectfont 
\pagestyle{plain}
\onecolumn

\begin{centering}

{\huge Extending the Harper Identity to Iterated Belief Change}

\vspace{2em}

\begin{tabular}{ c c c }
  {\Large Richard Booth}  & ~~~~~~~~~~~~~~~~~~~~  & {\Large Jake Chandler} \\
  {\large Cardiff University} & ~ & {\large La Trobe University} \\
  {\large Cardiff, UK} & ~ & {\large Melbourne, Australia} \\
  {\large boothr2@cardiff.ac.uk} & ~ & {\large jacob.chandler@latrobe.edu.au} \\  
\end{tabular}

\vspace{3.5em}

\end{centering}

\begin{quote}

{\large {\bf Abstract} The field of iterated belief change has focused mainly on {\em revision}, with the other main operator of AGM belief change theory, i.e., {\em contraction} receiving relatively little attention. In this paper we extend the Harper Identity from single-step change to define iterated contraction in terms of iterated revision. Specifically, just as the Harper Identity provides a recipe for defining the {\em belief set} resulting from contracting $A$ in terms of {\em (i)} the initial belief set and {\em (ii)} the belief set resulting from revision by $\neg A$, we look at ways to define the {\em plausibility ordering} over worlds resulting from contracting $A$ in terms of {\em (iii)} the initial plausibility ordering, and {\em (iv)} the plausibility ordering resulting from revision by $\neg A$. After noting that the most straightforward such extension leads to a trivialisation of the space of permissible orderings, we provide a family of operators for combining plausibility orderings that avoid such a result. These operators are characterised in our domain of interest  by a pair of intuitively compelling properties, which turn out to enable the derivation of a number of iterated contraction postulates from postulates for iterated revision. We finish by observing that a salient member of this family allows for the derivation of counterparts for contraction of some well known iterated revision operators, as well as for defining new iterated contraction operators. }
%\textcolor{red}{Recheck the abstract after we've finished. Should probably mention STQ as well.}
\end{quote}

~

\section{Introduction}

Since the publication of Darwiche and Pearl's seminal paper on the topic in the mid 90's \cite{darwiche1997logic}, a substantial body of research has now accumulated on the problem of iterated belief revision--the problem of how to adjust one's corpus of beliefs in response to a temporal sequence of successive additions to its members \cite{booth2006admissible,booth2011revise,boutilier1996iterated,jin2007iterated,nayak2003dynamic,peppas2014panorama}.

In contrast, work on the parallel problem of iterated {\em contraction}--the problem of how to adjust one's corpus in response to a sequence of successive retractions--was only initiated far more recently and remains comparatively underdeveloped \cite{chopra2008iterated,hansson2012global,hild2008measurement,nayak2006taking,nayak2007iterated,ramachandran2012three,rott2009shifting}.

One obvious way to level out this discrepancy would be to introduce a principle that enables us to derive, from constraints on iterated revision, corresponding constraints on iterated contraction. But while there exists a well known and widely accepted postulate connecting {\em single-shot} revision and contraction, the `Harper Identity' \cite{harper1976rational}, there has been no discussion to date of how to extend this principle to the iterated case.\footnote{\normalsize It should be noted that \cite{nayak2006taking} and Ramachandran {\em et al}  \cite{ramachandran2012three} do propose a principle that they call the `New Harper Identity'. But while this may be suggestive of an attempted extension of the Harper Identity to the iterated case, the New Harper Identity simply appears to be a representation, in terms of plausibility orderings, of a particular set of postulates for iterated contraction.
\vspace{0.25em}
}  
One idea, which we pursue in this paper, is that whereas the Harper Identity says the {\em belief set} resulting from contracting sentence $A$ should be formed by combining  {\em (i)} the initial belief set and {\em (ii)} the belief set resulting from revision by $\neg A$, we look for ways to define the {\em plausibility ordering} over worlds resulting from contracting $A$ in terms of {\em (iii)} the initial plausibility ordering, and {\em (iv)} the plausibility ordering resulting from revision by $\neg A$.
%The Harper Identity states that the belief set following contraction of a sentence $A$ can be obtained by intersecting the initial belief set with the belief set obtained by {\em revising} by $\neg A$

In the present paper, we first of all show that the simplest extension of the Harper Identity to iterated belief change is too strong a principle, being inconsistent with basic principles of belief dynamics on pains of  triviality      (Section 3). This leads us to consider a set of collectively weaker principles, which we show to characterise, in our domain of interest, a family of binary combination operators for total preorders that we call {\em TeamQueue} combinators (Section 4). After recapitulating a number of existing postulates from both iterated revision and contraction, we show how these two lists of postulates can be linked via the use of any TeamQueue combinator (Section 5). Then we prove some more specific results of this type using a particular TeamQueue combinator that we call {\em Synchronous TeamQueue} (Section 6). Finally we conclude and mention some ideas for future work. Proofs of the various propositions and theorems have been relegated to the appendix.

%In the present paper, we first of all show that the simplest proposal for extending the Harper Identity to iterated belief change falls prey to a triviality result. This leads us, in the following section, to try a different route whereby we look for ways to combine not just belief sets but rather the {\em plausibility orderings} (total preorders) over worlds that underlie each of the initial belief state and the belief state that results from revising by $\neg A$. For this purpose we introduce a family of binary combination operators for total preorders which we call {\em TeamQueue} combinators, and we give a characterisation of these operators in terms of postulates. After that we recap a few existing postulates from both iterated revision and contraction and we show how these two lists of postulates can be linked via the use of any TeamQueue combinator. Then we prove some more specific results of this type using a particular TeamQueue combinator that we call {\em Synchronous TeamQueue}. Finally we conclude and mention some ideas for future work.

%\clearpage

\section{Preliminaries}

We represent the beliefs of an agent by a so-called belief state $\Psi$, which we treat as a primitive. $\Psi$ determines a belief {\em set} $\bel{\Psi}$, a deductively closed set of sentences, drawn from a finitely generated propositional, truth-functional language $L$. The set of classical logical consequences of a sentence $A \in L$ is denoted by $\Cn{A}$. The set of propositional worlds is denoted by $W$, and the set of models of a given sentence $A$ is denoted by $\mods{A}$.

The dynamics of belief states are modelled by two operations--contraction and revision, which respectively return the posterior belief states $\Psi * A$ and $\Psi \contract A$ resulting from an adjustment of the prior belief state $\Psi$ to accommodate, respectively, the inclusion and exclusion of $A$. 

We assume that these operations satisfy the so-called AGM postulates \cite{alchourron1985logic}, which enforce a principle of `minimal mutilation' of the initial belief set in meeting the relevant exclusion or inclusion constraint. Regarding revision, we have:
\begin{tabbing}
BLAHBLI: \=\kill

(AGM$\ast$1) \> $\textrm{Cn}([\Psi*A])\subseteq [\Psi*A]$ \\[0.1cm]

(AGM$\ast$2) \> $A\in [\Psi* A]$\\[0.1cm]

(AGM$\ast$3) \> $[\Psi* A]\subseteq\textrm{Cn}([\Psi]\cup\{ A\})$\\[0.1cm]

(AGM$\ast$4) \> If $\neg A\notin [\Psi]$, then $\textrm{Cn}([\Psi]\cup\{ A\})\subseteq[\Psi* A]$\\[0.1cm]

(AGM$\ast$5) \>  If $A$ is consistent, then so too is $[\Psi*A]$\\[0.1cm]

(AGM$\ast$6) \> If $\textrm{Cn}(A)=\textrm{Cn}(B)$, then $[\Psi*A]=[\Psi*B]$\\[0.1cm]

(AGM$\ast$7) \> $[\Psi*(A\wedge B)]\subseteq\textrm{Cn}([\Psi*A]\cup\{B\})$\\[0.1cm]

(AGM$\ast$8) \> If $\neg B\notin [\Psi*A]$, then $\textrm{Cn}([\Psi*A]\cup\{B\})\subseteq [\Psi*(A\wedge B)]$\\[-0.25em]
\end{tabbing} 
\vspace{-0.5em}
Regarding contraction:
\begin{tabbing}
BLAHBLI: \=\kill

(AGM$\contract$1) \> $\textrm{Cn}([\Psi\contract A])\subseteq [\Psi\contract A]$\\[0.1cm]

(AGM$\contract$2) \> $[\Psi\contract  A]\subseteq[\Psi]$\\[0.1cm]

(AGM$\contract$3) \> If $A\notin [\Psi]$, then $[\Psi\contract A]=[\Psi]$\\[0.1cm]

(AGM$\contract$4) \> If $A\notin \textrm{Cn}(\varnothing)$, then $A\notin [\Psi\contract A]$\\[0.1cm]

(AGM$\contract$5) \>  If $A\in[\Psi]$, then $[\Psi]\subseteq\textrm{Cn}([\Psi\contract A]\cup\{ A\})$\\[0.1cm]

(AGM$\contract$6) \> If $\textrm{Cn}(A)=\textrm{Cn}(B)$, then $[\Psi\contract A]=[\Psi\contract B]$\\[0.1cm]

(AGM$\contract$7) \> $[\Psi\contract A]\cap[\Psi\contract B]\subseteq [\Psi\contract A\wedge B]$\\[0.1cm]

(AGM$\contract$8) \> If $A\notin [\Psi\contract A\wedge B]$, then $[\Psi\contract A\wedge B]\subseteq[\Psi\contract A]$\\[-0.25em]
\end{tabbing} 
\vspace{-0.5em}
We also assume that they 
%satisfy the DP postulates of \cite{darwiche1997logic}, namely:
%\begin{tabbing}
%BLAHBLI: \=\kill
%(C$\ast$1) \> If $A \in \mbox{Cn}(B)$ then $\bel{\Psi \ast A \ast B} = \bel{\Psi \ast B}$ \\[0.1cm]
%(C$\ast$2) \> If $\neg A \in \mbox{Cn}(B)$ then $\bel{\Psi \ast A \ast B} = \bel{\Psi \ast B}$ \\[0.1cm] 
%(C$\ast$3) \> If $A \in \bel{\Psi \ast B}$ then $A \in \bel{\Psi \ast A \ast B}$ \\[0.1cm]
%(C$\ast$4) \> If $\neg A \not\in \bel{\Psi \ast B}$ then $\neg A \not\in \bel{\Psi \ast A \ast B}$\\[-0.25em]
%\end{tabbing} 
%\vspace{-0.5em}
%and that they 
are linked in the one-step case by the Harper Identity (HI):
\begin{tabbing}
BLAHBLI: \=\kill
(HI)   \> $\bel{\Psi \contract A} =\bel{\Psi} \cap \bel{\Psi * \neg A}$ \\[-0.25em]
\end{tabbing} 
\vspace{-0.5em}
We follow a number of authors in making use of a `semantic' representation of the `syntactic' one-step revision and contraction dispositions associated with a particular belief state $\Psi$ in terms of a total preorder ({\em tpo}) $\preceq_{\Psi}$ over the set $W$ of possible worlds. Intuitively $\preceq_\Psi$ orders the worlds according to {\em plausibility} (with more plausible worlds lower down the ordering). Then the set $\min(\preceq_{\Psi}, \mods{A}):=\{x\in \mods{A}\mid \forall y\in \mods{A}, x\preceq_\Psi y\}$ of minimal $A$-worlds corresponds to the set of worlds in which all and only the sentences in $[\Psi*A]$ are true, with $\mods{\bel{\Psi}} = \min(\preceq_{\Psi}, W)$ for any $\Psi$ (see, for instance, the representation results in \cite{grove1988two,katsuno1991propositional}). Viewed in this way, the question of iterated belief change becomes a question about the dynamics of $\preceq_\Psi$ under contraction and revision, with HI translating into the constraint $\min(\preceq_{\Psi\contract A}, W) =$ \mbox{$\min(\preceq_{\Psi}, W)$} $\cup \min(\preceq_{\Psi\ast \neg A},W)$. We will denote the set of all tpos over $W$ by $T(W)$. The strict part of $\preceq_{\Psi}$ will be denoted by $\prec_{\Psi}$ and its symmetric part \mbox{by $\sim_{\Psi}$.} 

A tpo $\preceq_{\Psi}$ can also be represented by an ordered partition $\langle S_1, S_2, \ldots S_m\rangle$ of $W$, with $x \preceq_{\Psi} y$ iff $r(x, \preceq_{\Psi}) \leq$ \mbox{$r(y, \preceq_{\Psi})$,} where $r(x, \preceq_{\Psi})$ denotes the `rank' of $x$ with respect to $\preceq_{\Psi}$ and is defined by taking $S_{r(x, \preceq_{\Psi})}$ to be the cell in the partition that contains $x$. 
%Alternatively, $\preceq_{\Psi}$ can be represented by a set of `Ramsey test conditionals', whose members are drawn from an extension $L_1$ of $L$ to include all sentences of the form $A>B$, where $A,B\in L$, and which satisfies the so called Ramsey test:
%\begin{tabbing}
%BLAHBLI: \=\kill
%(RT)   \> $A > B \in \bel{\Psi}_1\ \textrm{iff}\ B \in \bel{\Psi * A}$ \\[-0.25em]
%\end{tabbing} 
%\vspace{-0.5em}
%where $\bel{\Psi * A}_1$ denotes the `extended belief set' associated with $\Psi$, that subsumes $[\Psi]$ and includes the relevant Ramsey test conditionals. 
%\textcolor{red}{Is this line really necessary? RT still makes sense without it. SORRY FOR THE CAPS (!): YES, I THINK WE DO, SINCE $\Psi$ IS TREATED AS A PRIMITIVE AND $\bel{\Psi}$ SENTENCES FROM THE BOOLEAN LANGUAGE. I ALSO MADE A MISTAKE IN THE ORIGINAL FORMULATION}

%
%\textcolor{red}{Iterated belief change is all about modelling how the tpo changes. This has been how many papers on iterated revision have oriented themselves, and also most of the papers that exist on iterated contraction up to now.  }

\section{A triviality result}

What should an agent believe after performing a contraction followed by a revision? We would like to extend the Harper Identity to cover this case. 

In syntactic terms, the most straightforward suggestion would be to simply extend HI to cover not just one's beliefs, but also one's commitments to retain or lose various beliefs upon subsequent revisions:
\begin{tabbing}
BLAHBLI: \=\kill
(EHI)   \> $\bel{(\Psi \contract A) \ast B}=\bel{\Psi \ast B} \cap \bel{(\Psi \ast \neg A) \ast B}$ \\[-0.25em]
\end{tabbing} 
\vspace{-0.5em}
If $B \equiv \top$ then we obtain HI as a special case. Note that under weak assumptions, EHI can equivalently be restated in terms of contraction only:
\begin{proposition}\label{EHICtoR}
EHI entails 
\begin{tabbing}
BLAHBLI: \=\kill
(EHIC)   \> $[(\Psi \contract A) \contract B] = [\Psi]\cap[\Psi * \neg B] \cap [\Psi*\neg A]\cap [(\Psi * \neg A) *\neg B]$ \\[-0.25em]
\end{tabbing} 
 \vspace{-0.5em}
and is equivalent to it in the presence of AGM$\ast$3 and the Levi Identity:
\begin{tabbing}
BLAHBLI: \=\kill
(LI)   \>  $[\Psi*A]=\Cn{[\Psi\contract \neg A] \cup \{A\}}$.\\[-0.25em]
\end{tabbing} 
\end{proposition}
 \vspace{-0.5em}
However, as G\"ardenfors' classic triviality result and its subsequent refinements \cite{gardenfors1986belief,rott1989conditionals,etlin2009problem} have taught us, the unqualified extension of principles of belief dynamics to cover conditional beliefs is a risky business. And as it turns out, the above proposal is too strong: it can be shown that, under mild constraints on single shot revision and contraction, it places unacceptable restrictions on the space of permissible belief sets resulting from single revisions:
\begin{proposition}\label{EHICtriv}

In the presence of AGM$\ast$5,  AGM$\ast$6 and AGM$\contract$3, EHI (and more specifically, HI, alongside its left-to-right half  $\bel{(\Psi \contract A) \ast B}\subseteq\bel{\Psi \ast B} \cap \bel{(\Psi \ast \neg A) \ast B}$) entails that there does not exist a belief state $\Psi$ such that: (i) $[\Psi]=\mbox{Cn}(p\wedge q)$, (ii) $[\Psi*\neg p]=\mbox{Cn}(\neg p\wedge q)$ and (iii) $[\Psi *p\leftrightarrow\neg q]=\mbox{Cn}(p\leftrightarrow\neg q)$, where $p$ and $q$ are propositional atoms.\footnote{ \normalsize  The problem that we have just noted for EHI is closely related to the observation that an intersection of two sets of `rational doxastic conditionals' need not itself be rational, which is familiar from the literature on default reasoning \cite{lehmann1992does}.}

\end{proposition}

\noindent The above strategy and its shortcomings can equivalently be recast in semantic terms. Let us call a function $\oplus$ that takes pairs of tpos as inputs and yields a tpo as an output a tpo combination operator, or a  `{\em combinator}'. For convenience, we denote  $\preceq_1\! \oplus\! \preceq_2$ by `$\combi$'. 

In extending the Harper Identity to the iterated case, we are essentially looking for an appropriate combinator  $\oplus$ such that:
\begin{tabbing}
BLAHBLI: \=\kill
(COMBI) \> $\preceq_{\Psi \contract A}=\preceq_{\Psi}\!\oplus\! \preceq_{\Psi \ast\neg A}$\\
\end{tabbing} 
\vspace{-0.5em}
Now, just as HI corresponds, given COMBI, to the following semantic principle:
\begin{tabbing}
BLAHBLI: \=\kill
($\oplus$HI) \> $\min(\combi, W) = \min(\preceq_1, W) \cup \min(\preceq_2, W)$ \\[-0.25em]
\end{tabbing} 
\vspace{-0.5em}
EHI amounts to 
\begin{tabbing}
BLAHBLI: \=\kill
($\oplus$EHI) \> For all $S\subseteq W$, $\min(\combi, S)=\min(\preceq_1, S)\cup\min(\preceq_2, S)$  \\[-0.25em]
\end{tabbing} 
\vspace{-0.5em}
What our result above effectively demonstrates is that no combinator $\oplus$ satisfies $\oplus$EHI unless we place undesirable restrictions on its domain: $\oplus$EHI is too much to ask for.

We will continue approaching our issue of interest from a predominantly  semantic perspective for the remainder of the paper. In the following section, we retreat from $\oplus$EHI to offer an altogether weaker set of minimal postulates for $\oplus$, before taking a look at a concrete family of `Team Queuing' combinators that satisfy them. We first establish a general characterisation of this family before showing that our set of minimal postulates suffices to characterise it in our restricted domain of interest.

\section{Combinators: the bottom line}

Since we are in the business of extending the Harper Identity, we will begin  by requiring satisfaction of $\oplus$HI. We call combinators that satisfy this property `{\em basic}' combinators. 

In addition, even though EHI is too strong, certain weakenings of it do seem to be compelling. Specifically, it seems appropriate to require that  our combination method leads to the following weak lower and upper  bound principles:
\begin{tabbing}
BLAHBLI: \=\kill
(LB)   \> $\bel{\Psi \ast B} \cap \bel{(\Psi \ast \neg A) \ast B} \subseteq
\bel{(\Psi \contract A) \ast B}$ \\[0.1cm] 
(UB) \> $\bel{(\Psi \contract A) \ast B} \subseteq \bel{\Psi \ast B} \cup \bel{(\Psi \ast \neg A) \ast B}$\\[-0.25em]
\end{tabbing} 
\vspace{-0.5em}
We note that the former corresponds to the half of EHI that was {\em not} implicated in our earlier triviality result. Given COMBI, these will be ensured by requiring, respectively, the following upper and lower bounds on $\min(\combi, S)$ for any $S \subseteq W$ (note an upper, resp.\ lower bound on world-sets yields a lower, resp.\ upper bound on belief sets):
\begin{tabbing}
BLAHBLI: \=\kill
($\oplus$UB) \> $\min(\combi, S) \subseteq \min(\preceq_1, S) \cup \min(\preceq_2, S)$ \\[0.1cm] 
($\oplus$LB)   \> Either $\min(\preceq_1, S) \subseteq \min(\combi, S)$ or  $\min(\preceq_2, S) \subseteq \min(\combi, S)$ \\[-0.25em]
\end{tabbing} 
\vspace{-0.5em}

%\begin{tabbing}
%BLAHBLI: \=\kill
%(LB)   \> $\bel{\Psi \ast B} \cap \bel{(\Psi \ast \neg A) \ast B} \subseteq
%\bel{(\Psi \contract A) \ast B}$ \\[0.1cm]
%(UB) \> $\bel{(\Psi \contract A) \ast B} \subseteq \bel{\Psi \ast B} \cup \bel{(\Psi \ast \neg A) \ast B}$\\[-0.25em]
%\end{tabbing} 
%\vspace{-0.5em}
%We note that LB corresponds to the half of EHI,that was not implicated in our earlier triviality result.
%These can alternatively be recast, using RT, in terms of Ramsey test conditionals, in the form of the following `preservation of common conditionals' and `no new conditionals' properties:
%\begin{tabbing}
%BLAHBLI: \=\kill
%(PCC)   \> If $B>C\in \Psi$ and $B>C\in (\Psi \ast \neg A)$, \\
%	\> then $B>C$ $\in (\Psi \contract A)$   \\[0.1cm]
%(NNC)   \> If $B>C\in (\Psi \contract A)$, then either  $B>C \in \Psi$\\
%	\>  or $B>C\in  (\Psi \ast \neg A)$ \\[-0.25em]
%\end{tabbing} 
%\vspace{-0.5em}
\noindent
$\oplus$UB and $\oplus$LB can be repackaged using only binary comparisons:
\begin{proposition}\label{SWPUandULB}
$\oplus$UB and $\oplus$LB are respectively equivalent to the following:
\begin{tabbing}
BLAHBLI: \=\kill
($\oplus$SPU+)   \> If $x \prec_1 y$ and $z \prec_2 y$ then either $x \prec_{1 \oplus 2} y$ or  $z \prec_{1 \oplus 2} y$  \\[0.1em]
%\end{tabbing} 
%\vspace{-1em}
%and
%\begin{tabbing}
%BLAHBLI: \=\kill
($\oplus$WPU+)   \> If $x \preceq_1 y$ and $z \preceq_2 y$ then either $x \combi y$ or  $z \combi y$ \\[-0.25em]
\end{tabbing} 
\end{proposition} 
\vspace{-0.5em}
$\oplus$SPU+ and $\oplus$WPU+ owe their names to their being respective strengthenings of the following principles of strict and weak preference unanimity, which are analogues of  the `weak Pareto' and `Pareto weak preference' principles found in the social choice literature:
\begin{tabbing}
BLAHBLI: \=\kill
($\oplus$SPU)   \> If $x \prec_1 y$ and $x \prec_2 y$ then $x \prec_{1 \oplus 2} y$  \\[0.1cm]
($\oplus$WPU)   \> If $x \preceq_1 y$ and $x \preceq_2 y$ then $x \combi y$  \\[-0.25em]
\end{tabbing} 
\vspace{-0.5em}

%\subsection{TeamQueue Combinators}

\noindent We now consider a concrete family of basic combinators that satisfy both  $\oplus$SPU+ and $\oplus$WPU+, and, indeed, can be shown to be characterised by precisely these principles in our domain of interest. We call these `TeamQueue' combinators. 

The basic idea behind this family--and motivation behind the name given to it--can be grasped by means of the following analogy:  A number of couples go shopping for groceries. The supermarket that they frequent is equipped with two tills. For each till, we find a sequence of various groups of people queueing to pay for their items. In order to minimise the time spent in the store, each couple operates by ``team queueing'': each member of the pair joins a group in a different queue and leaves their place to join their partner's group in case this group arrives at the till first. After synchronously processing their first group of customers, the tills may or may not then  operate at different and variable speeds. We consider the temporal sequence of sets of couples leaving the store. In our setting, the queues are the two tpos (with lower elements towards the head of the queue) and the couples are pairs of copies of each world.

More formally, we assume, for each ordered pair 
\mbox{$\langle\preceq_1, \preceq_2\rangle$} of tpos, a sequence $\langle a_{\preceq_1, \preceq_2}(i)\rangle_{i \in \mathbb{N}}$ such that:
\vspace{0.2cm} 
\\
\begin{tabular}{@{}ll}
$(a1)$ & $\emptyset \neq a_{\preceq_1, \preceq_2}(i) \subseteq \{ 1, 2\}$ for each $i$, 
\vspace{0.05cm}
%\\
%$(a2)$ & $1 \in a_{\preceq_1, \preceq_2}(i)$ iff $2 \in a_{\preceq_2, \preceq_1}(i)$ 
%\vspace{0.05cm}
\\
$(a2)$ & $a_{\preceq_1, \preceq_2}(1) = \{1,2\}$
\end{tabular}
\vspace{0.2cm}
\\
$a_{\preceq_1, \preceq_2}(i)$ specifies which queue is to be processed at each step. Then $(a1)$ ensures either one or both are processed, and $(a2)$ says both are processed at the initial stage (which will ensure $\oplus$HI holds for the resulting combinators).
Then we construct the ordered partition $\langle T_1, T_2, \ldots, T_m\rangle$ corresponding to $\preceq_{1\oplus 2}$ inductively as follows:
\[
T_{i} = \bigcup_{j \in a_{\preceq_1, \preceq_2}(i)} \min(\bigcap_{k<i}T_{k}^c, \preceq_j)
\]
(where `$T^c$' denotes the complement of set $T$)
and $m$ is minimal such that $\bigcup_{i\leq m} T_i = W$. With this in hand, we can now offer:
\begin{definition}
$\oplus$ is a {\em TeamQueue combinator} iff, for each ordered pair $\langle\preceq_1, \preceq_2\rangle$ of tpos there exists a sequence $\langle a_{\preceq_1, \preceq_2}(i)\rangle_{i \in \mathbb{N}}$ satisfying (a1) and (a2) such that  $\preceq_{1\oplus 2}$ is obtained as above.
\end{definition}
%\textcolor{blue}{If you don't care about enforcing commutativity then you can drop $(a2)$.}
It is easily verified that TeamQueue combinators are indeed basic combinators. The following example provides an elementary illustration of the combinator at work: 
\begin{example}
Suppose that $W = \{w,x,y,z\}$, that $\preceq_1$ is the tpo represented by the ordered partition $\langle \{z\}, \{w\}, \{x,y\}\rangle$,  and that  $\preceq_2$ is represented by $\langle\{x,z\}, \{y\}, \{w\}\rangle$.~Let  $\oplus $ be a TeamQueue combinator such that $\langle a_{\preceq_1, \preceq_2}(i)\rangle_{i \in \mathbb{N}} = \langle\{1,2\}, \{2\}, \{1\},\ldots\rangle$. Then the ordered partition corresponding to $\combi$ is $\langle T_1, T_2, T_3\rangle = \langle \{x,z\}, \{y\}, \{w\}\rangle$, since 
\begin{eqnarray*}
T_1 & = & \bigcup_{j \in \{1,2\}} \min(W, \preceq_j) = \{x,z\} \\
T_2 & = & \min(T_{1}^c, \preceq_2) = \{y\} \\
T_3 & = & \min(T_{1}^c \cap T_{2}^c, \preceq_1) = \{w\}
\end{eqnarray*}
\end{example}
As noted above, TeamQueue combinators satisfy both  $\oplus$SPU+ and $\oplus$WPU+. In fact, one can show that this family can actually be characterised by these two conditions, in the presence of a third:
\begin{theorem}
$\oplus$ is a TeamQueue combinator iff it is a basic combinator that satisfies $\oplus$SPU+, $\oplus$WPU+ and the following `no overtaking' property;
%\begin{tabbing}
%BLAHBLI: \=\kill
%($\oplus$NO)   \> If either (i) $x \preceq_1 y$ and $z \prec_2 y$ or (ii) $x \preceq_2 y$ and\\
%\> $z \prec_1 y$, then either $x \combi y$ or $z \prec_{1\oplus 2} y$ \\[-0.25em]
%\end{tabbing} 
%\end{proposition}
%\vspace{-0.5em}
\begin{tabbing}
BLAHBLI: \=\kill
($\oplus$NO)   \> For $i\neq j$, if $x\prec_i y$ and $z\preceq_j y$, then either $x\prec_{1\oplus 2} y$ or $z\preceq_{1\oplus 2} y$  \\[-0.25em]
\end{tabbing} 
\end{theorem}
\vspace{-0.5em}

%
%\textcolor{red}{Strictly speaking there are 2 NO properties. The above one and the one where the subscripts 1 and 2 are switched. They are equivalent only in the presence of commutaivity. TRUE. I HAVE NOW BUILT  COMMUTATIVITY INTO THE DEF OF TQ TO SIMPLIFY THINGS}
%\begin{proof}
%{\em (Outline of completeness part)} From any basic combinator $\oplus$ satisfying the postulates we can construct, for each ordered pair $\langle \preceq_1, \preceq_2 \rangle$ of tpos, a sequence $\langle a_{\preceq_1, \preceq_2}(i)\rangle_{i \in \mathbb{N}}$. Assuming $\preceq_{1 \oplus 2}$ is represented by $\langle S_1, \ldots,S_n \rangle$, we set 
%\[
%j \in a_{\preceq_1, \preceq_2}(i)\
%\textrm{iff }
%\min(\preceq_j, \bigcap_{k < i}S_k^c)
%\subseteq 
%S_i. 
%\]
%$\langle a_{\preceq_1, \preceq_2}(i)\rangle_{i \in \mathbb{N}}$ satisfies $(a1)$ by property $\oplus$TRI below
%%, satisfies $(a2)$ by commutativity, 
%and satisfies $(a2)$ from $\oplus$HI. Let $\oplus'$ denote the TeamQueue combinator defined by this sequence. Letting $\langle T_1,\ldots, T_m \rangle$ represent $\preceq_{1\oplus'2}$, it is then a straightforward inductive proof to show $T_i = S_i$ for all $i$. QED
%\end{proof}
%
%\vspace{0.2cm}

\noindent
Taken together, the three postulates  $\oplus$SPU+, $\oplus$WPU+ and $\oplus$NO say that in $\preceq_{1\oplus 2}$, no world $x$ is allowed to improve its position w.r.t.\ {\em both} input orderings $\preceq_1$ and $\preceq_2$. Indeed each postulate blocks one of the three possible ways in which this `no double improvement' condition could be violated. We note that this condition can be cashed out in terms of the following remarkable property: 
\begin{proposition}\label{tri}
$\oplus$ is a TeamQueue combinator iff it is a basic combinator that satisfies the following `trifurcation' property, for all $S\subseteq W$:
\begin{tabbing}
BLAHBLI: \=\kill
($\oplus$TRI)   \> $\min(\combi, S)$ is equal to either $\min(\preceq_1, S)$, $\min(\preceq_2, S)$ or $\min(\preceq_1, S) $\\
\>$\cup \min(\preceq_2, S)$ \\[-0.25em]
\end{tabbing} 
\end{proposition}
\vspace{-0.5em}
Given COMBI, $\oplus$TRI yields the claim that $\bel{(\Psi \contract A) \ast B}$ is equal to either $\bel{\Psi \ast B}$, $\bel{(\Psi \ast \neg A) \ast B}$ or $\bel{\Psi \ast B} \cap \bel{(\Psi \ast \neg A) \ast B}$.

%\begin{tabbing}
%BLAHBLI: \=\kill
%(TRI)   \> $\bel{(\Psi \contract A) \ast B}$ is equal to either $\bel{\Psi \ast B}$, \\
%	\>$\bel{(\Psi \ast \neg A) \ast B}$ or $\bel{\Psi \ast B} \cap \bel{(\Psi \ast \neg A) \ast B}$ \\[-0.25em]
%\end{tabbing} 
%\end{proposition}
%\vspace{-0.5em}
%\subsection{TeamQueue under domain restriction}
To wrap up this section, it should be noted that the results so far have been perfectly domain-general, in the sense that they hold for combinators whose domain corresponded to the entire space of pairs of tpos defined over $W$. Our problem of interest is somewhat narrower in scope, however, since we are interested in the special case in which one of the tpos is obtained from the other by means of a revision. In particular, we assume the first two semantic postulates of \cite{darwiche1997logic}. 
\begin{tabbing}
BLAHBLI: \=\kill
(CR$\ast$1)  \> If $x,y \in \mods{A}$ then $x \preceq_{\Psi \ast A} y$ iff  $x \preceq y$\\[0.1cm]
(CR$\ast$2) \> If $x,y \in \mods{\neg A}$ then $x \preceq_{\Psi \ast A} y$ iff  $x \preceq y$\\[-0.25em]
\end{tabbing} 
\vspace{-0.5em}
%
%\vspace{0.2cm}
%\\
%\begin{tabular}{@{}ll}
%(CR$\ast$1) & If $x,y \in \mods{A}$ then $x \preceq_{\Psi \ast A} y$ iff  $x \preceq y$
%\vspace{0.05cm} \\
%(CR$\ast$2) & If $x,y \in \mods{\neg A}$ then $x \preceq_{\Psi \ast A} y$ iff  $x \preceq y$
%\end{tabular}
%\vspace{0.2cm}
%\\
In other words, $\preceq_1$ and $\preceq_2$ will always be $\mods{A}$-variants for some sentence $A$, in the following sense:
\begin{definition}
Given $\preceq_1, \preceq_2 \in T(W)$ and $S \subseteq W$, we say $\preceq_1$ and $\preceq_2$ are {\em $S$-variants} iff [$x\preceq_1 y$ iff $x \preceq_2 y$] holds for all $x, y \in (S \times S) \cup (S^c \times S^c)$. We let $V(W)$ denote the set of all $\langle \preceq_1, \preceq_2 \rangle \in T(W) \times T(W)$ such that $\preceq_1, \preceq_2$ are $S$-variants for some $S \subseteq W$.
\end{definition}

\begin{example}
Suppose that $W = \{w,x,y,z\}$, that $\preceq_1$ is the tpo represented by the ordered partition $\langle\{w\}, \{x\}, \{y\}, \{z\}\rangle$, and that $\preceq_2$ is represented by $\langle\{w\}, \{x,y\}, \{z\}\rangle$. Then $\preceq_1$ and $\preceq_2$ are $\{y,z\}$-variants, since (i) $w\prec_1 x$ and $w\prec_2 x$, as well as (ii) $y\prec_1 z$ and $y\prec_2 z$. They are not, however, $\{x,y\}$-variants, since $x\prec_1 y$ but $y\preceq_2 x$. 
\end{example}

\noindent This leads to the following domain restriction on $\oplus$: 
\vspace{0.2cm}
\\
\begin{tabular}{@{}ll}
($\oplus$DOM) & $\textit{Domain}(\oplus) \subseteq V(W)$
\end{tabular}
\vspace{0.2cm}
\\
%\begin{tabbing}
%BLAHBLI: \= BL\=\kill
%(DOM)   \> There exists $S \subseteq W$ such that\\
%\> \> - For all $x, y\in S$, $x \preceq_1 y$ iff $x \preceq_2 y$\\
%\> \> - For all $x, y \in S^c$, $x \preceq_1 y$ iff $x \preceq_2 y$\\[-0.25em]
%\end{tabbing} 
%\vspace{-0.5em}
As it turns out, this constraint allows for a noteworthy simplification of the characterisation of TeamQueue combinators:
\begin{proposition}\label{TQ}
Given $\oplus \textit{DOM}$, $\oplus$ is a TeamQueue combinator iff it is a basic combinator that satisfies $\oplus$SPU+ and $\oplus$WPU+.
\end{proposition}
%\textcolor{red}{Isn't the more direct to say NO is derivable from SPU+ and WPU+ given ($\oplus \textit{DOM}$)?}
We also note, in passing, that 
\begin{proposition}\label{SWPU+andSWPU}
Given $\oplus \textit{DOM}$, $\oplus$ satisfies $\oplus$SPU+ and $\oplus$WPU+ iff it satisfies $\oplus$SPU and $\oplus$WPU, respectively.
\end{proposition}
Given Proposition \ref{tri}, the potentially surprising upshot of Proposition \ref{TQ} is that, in our domain of interest, satisfaction of $\oplus$LB and $\oplus$UB entails satisfaction of $\oplus$TRI. 

\section{Iterated Contraction via TeamQueue Combination}

A central result of AGM theory says that, under assumption of HI, if  $\ast$ satisfies the AGM revision postulates, then $\contract$ automatically satisfies the AGM contraction postulates.
In this section we look at some of the postulates for both iterated revision and contraction that have been proposed in the literature. We show that, if $\preceq_{\Psi \contract A}$ is defined from $\preceq$ and $\preceq_{\Psi \ast \neg A}$ using COMBI via a TeamQueue combinator, then satisfaction of some well known sets of postulates for iterated revision  leads to satisfaction of other well known sets of postulates for iterated contraction. 

%\subsection{Postulates for Iterated Revision}

The most widely cited postulates for iterated revision are the four DP postulates of \cite{darwiche1997logic}. These, like most of the postulates for iterated belief change, come in two flavours: a {\em semantic} one in terms of requirements on the tpo $\preceq_{\Psi \ast A}$ associated to the revised state $\Psi \ast A$, and a {\em syntactic} one in terms of requirements on the belief set $[(\Psi \ast A) \ast B]$ following a sequence of two revisions. Turning first to the semantic versions, we've already encountered the first two of these postulates--CR$\ast$1 and CR$\ast$2--in the previous section. The other two are
%\vspace{0.2cm}
%\\
%\begin{tabular}{@{}ll}
%(CR$\ast$3) & If $x \in \mods{A}$, $y \in \mods{\neg A}$ and $x \prec y$ then $x \prec^\ast_A y$
%\vspace{0.05cm} \\
%(CR$\ast$4) & If $x \in \mods{A}$, $y \in \mods{\neg A}$ and $x \preceq y$ then $x \preceq_{\Psi \ast A} y$ 
%\end{tabular}
%\vspace{0.2cm}
%\\
\begin{tabbing}
BLAHBLI: \=\kill
(CR$\ast$3) \> If $x \in \mods{A}$, $y \in \mods{\neg A}$ and $x \prec y$ then $x \prec_{\Psi \ast A}y$ \\[0.1cm]

(CR$\ast$4) \> If $x \in \mods{A}$, $y \in \mods{\neg A}$ and $x \preceq y$ then $x \preceq_{\Psi \ast A} y$ \\[-0.25em]
\end{tabbing} 
\vspace{-0.5em}
Each of these has an equivalent corresponding syntactic version as follows:
%\vspace{0.2cm}
%\\
%\begin{tabular}{@{}ll}
%(C$\ast$1) & If $A \in \mbox{Cn}(B)$ then $\bel{(\Psi \ast A) \ast B} = \bel{\Psi \ast B}$ 
%\vspace{0.05cm} \\
%(C$\ast$2) & If $\neg A \in \mbox{Cn}(B)$ then $\bel{(\Psi \ast A) \ast B} = \bel{\Psi \ast B}$ 
%\vspace{0.05cm} \\
%(C$\ast$3) & If $A \in \bel{\Psi \ast B}$ then $A \in \bel{(\Psi \ast A) \ast B}$ 
%\vspace{0.05cm} \\
%(C$\ast$4) & If $\neg A \not\in \bel{\Psi \ast B}$ then $\neg A \not\in \bel{(\Psi \ast A) \ast B}$
%\end{tabular}
%\vspace{0.1cm}
%\\
\begin{tabbing}
BLAHBLI: \=\kill
(C$\ast$1) \> If $A \in \mbox{Cn}(B)$ then $\bel{(\Psi \ast A) \ast B} = \bel{\Psi \ast B}$  \\[0.1cm]

(C$\ast$2) \> If $\neg A \in \mbox{Cn}(B)$ then $\bel{(\Psi \ast A) \ast B} = \bel{\Psi \ast B}$ \\[0.1cm]

(C$\ast$3) \> If $A \in \bel{\Psi \ast B}$ then $A \in \bel{(\Psi \ast A) \ast B}$  \\[0.1cm]

(C$\ast$4) \> If $\neg A \not\in \bel{\Psi \ast B}$ then $\neg A \not\in \bel{(\Psi \ast A) \ast B}$\\[-0.25em]
\end{tabbing} 
\vspace{-0.5em}
%\textcolor{red}{``We have the strong vacuity principle, according to which if the sentence to be added is already believed, then there is no need to change the state $\Psi$. In particular the revision function remains unaltered:
%\vspace{0.2cm}
%\\
%\begin{tabular}{@{}ll}
%(SV$\ast$) & If $A \in \bel{\Psi}$ then $\bel{\Psi \ast A \ast B} = \bel{\Psi \ast B}$ 
%\end{tabular}
%\vspace{0.2cm}
%\\
%Some of the above mentioned iterated revision operators don't satisfy this, although they can easily be modified to conform to it by adding a separating clause for the special case.'' Do we really need VAC? Maybe not.}

%\textcolor{red}{How about mentioning natural, restrained and lex revision here?}

%\subsection{Postulates for Iterated Contraction}
\noindent Chopra et al \shortcite{chopra2008iterated} proposed a list of `counterparts' to the DP postulates for the case of $\Psi \contract A$. The semantic versions of these were:
%\vspace{0.1cm}
%\\
%\begin{tabular}{@{}ll}
%(CR$\contract$1) & If $x,y \in \mods{\neg A}$ then $x \preceq_{\Psi \contract A} y$ iff  $x \preceq y$
%\vspace{0.1cm} \\
%(CR$\contract$2) & If $x,y \in \mods{A}$ then $x \preceq_{\Psi \contract A} y$ iff  $x \preceq y$ 
%\vspace{0.1cm} \\
%(CR$\contract$3) & If $x \in \mods{\neg A}$, $y \in \mods{A}$ and $x \prec y$ then $x \prec^\contract_A y$
%\vspace{0.1cm} \\
%(CR$\contract$4) & If $x \in \mods{\neg A}$, $y \in \mods{A}$ and $x \preceq y$ then $x \contract^\ast_A y$ \end{tabular}
%\vspace{0.1cm}
%\\
\begin{tabbing}
BLAHBLI: \=\kill
(CR$\contract$1)  \>  If $x,y \in \mods{\neg A}$ then $x \preceq_{\Psi \contract A} y$ iff  $x \preceq y$ \\[0.1cm]

(CR$\contract$2) \> If $x,y \in \mods{A}$ then $x \preceq_{\Psi \contract A} y$ iff  $x \preceq y$\\[0.1cm]

(CR$\contract$3)  \> If $x \in \mods{\neg A}$, $y \in \mods{A}$ and $x \prec y$ then $x \prec_{\Psi \contract A} y$ \\[0.1cm]

(CR$\contract$4) \> If $x \in \mods{\neg A}$, $y \in \mods{A}$ and $x \preceq y$ then $x \preceq_{\Psi \contract A} y$\\[-0.25em]
\end{tabbing} 
\vspace{-0.5em}
Chopra et al \shortcite{chopra2008iterated} showed (their Theorem 2) that, in the presence of the AGM postulates (reformulated as in our setting to apply to belief {\em states} rather than just belief {\em sets}) each of these postulates has an equivalent syntactic version as follows:
%\vspace{0.2cm}
%\\
%\begin{tabular}{@{}ll}
%(C$\contract$1) & If $\neg A \in \mbox{Cn}(B)$ then $\bel{(\Psi \contract A) \ast B} = \bel{\Psi \ast B}$ 
%\vspace{0.1cm} \\
%(C$\contract$2) & If $A \in \mbox{Cn}(B)$ then $\bel{(\Psi \contract A) \ast B} = \bel{\Psi \ast B}$ 
%\vspace{0.1cm} \\
%(C$\contract$3) & If $\neg A \in \bel{\Psi \ast B}$ then $\neg A \in \bel{(\Psi \contract A) \ast B}$ 
%\vspace{0.1cm} \\
%(C$\contract$4) & If $A \not\in \bel{\Psi \ast B}$ then $A \not\in \bel{(\Psi \contract A) \ast B}$
%\end{tabular}
%\vspace{0.2cm}
%\\
\begin{tabbing}
BLAHBLI: \=\kill
(C$\contract$1)  \>  If $\neg A \in \mbox{Cn}(B)$ then $\bel{(\Psi \contract A) \ast B} = \bel{\Psi \ast B}$ \\[0.1cm]

(C$\contract$2) \> If $A \in \mbox{Cn}(B)$ then $\bel{(\Psi \contract A) \ast B} = \bel{\Psi \ast B}$ \\[0.1cm]

(C$\contract$3)  \> If $\neg A \in \bel{\Psi \ast B}$ then $\neg A \in \bel{(\Psi \contract A) \ast B}$ \\[0.1cm]

(C$\contract$4) \> $A \not\in \bel{\Psi \ast B}$ then $A \not\in \bel{(\Psi \contract A) \ast B}$\\[-0.25em]
\end{tabbing} 
\vspace{-0.5em}
%\textcolor{red}{``We also have the Strong Vacuity Principle, according to which if sentence is not believed to begin with, then removing has no effect on the revision function.
%\vspace{0.2cm}
%\\
%\begin{tabular}{@{}ll}
%(SV$\contract$) & If $A \not\in \bel{\Psi}$ then $\bel{\Psi \contract A \ast B} = \bel{\Psi \ast B}$ 
%\end{tabular}
%\vspace{0.2cm}
%\\
%'' Again: do we need this?}
%\textcolor{red}{Can we reduce spacing between items slightly in these postulate lists?}
As it turns out, the definition of $\preceq_{\Psi \contract A}$ from $\preceq$ and $\preceq_{\Psi \ast \neg A}$ using COMBI via a TeamQueue combinator allows us to show the precise sense in which Chopra {\em et al}'s postulates are `Darwiche-Pearl-like', as they put it:
\begin{proposition}
Let $\oplus$ be a TeamQueue combinator, let $\ast$ be an AGM revision operator and let $\contract$ be such that $\preceq_{\Psi \contract A}$ is defined from $\ast$ via COMBI using $\oplus$. Then, for each $i = 1,2,3,4$, if $\ast$ satisfies CR$\ast i$ then $\contract$ satisfies CR$\contract i$.
\end{proposition}
As a corollary, given the AGM postulates, we recover the same result for the syntactic versions as well.

Finally, Nayak {\em et al} \shortcite{nayak2007iterated} have endorsed the following principle of `Principled Factored Intersection', which they show to be satisfied by a number of proposals for iterated contraction:
\begin{tabbing}
BLAHBL \= BLA\=\kill
(PFI)   \> Given $B\in\bel{\Psi\contract A}$\\[0.1cm]

\>  (a) \> If $\neg B \rightarrow \neg A\in [(\Psi \contract A) \contract  B]$, then $[(\Psi \contract A) \contract  B] =[\Psi \contract A]\cap$\\
\>\>$[\Psi\contract \neg A\rightarrow B]$\\[0.1cm]
 \>  (b) \> If $\neg B \rightarrow \neg A,\neg B\rightarrow A\notin [(\Psi \contract A) \contract  B]$, then $[(\Psi \contract A) \contract  B] =$\\
\>\>$[\Psi \contract A]\cap[\Psi\contract \neg A\rightarrow B]\cap[\Psi\contract A\rightarrow B]$\\[0.1cm]
 \>  (c) \> If $\neg B\rightarrow A\in [(\Psi \contract A) \contract  B]$, then  $[(\Psi \contract A) \contract  B] =[\Psi \contract A]\cap$\\
\>\>$[\Psi\contract  A\rightarrow B]$\\[-0.25em]
\end{tabbing} 
\vspace{-0.5em}
%
%\begin{exe}[align=left]
%
%\exi{\bf PFI}  \begin{xlist}
%\exi{\bf (a)} If $\neg B \rightarrow \neg A\in [(\Psi \contract A) \contract  B]$,\\
%then $[(\Psi \contract A) \contract  B] =[\Psi \contract A]\cap[\Psi\contract \neg A\rightarrow B]$
%\exi{\bf (b)} If $\neg B \rightarrow \neg A,\neg B\rightarrow A\notin [(\Psi \contract A) \contract  B]$,\\
%then $[(\Psi \contract A) \contract  B] =[\Psi \contract A]\cap[\Psi\contract \neg A\rightarrow B]\cap[\Psi\contract A\rightarrow B]$
%\exi{\bf (c)} If $\neg B\rightarrow A\in [(\Psi \contract A) \contract  B]$, \\
%then $[(\Psi \contract A) \contract  B] =[\Psi \contract A]\cap[\Psi\contract  A\rightarrow B]$
%\end{xlist}
%
%\end{exe}
The rationale for PFI remains rather unclear to date. Indeed, the only justifications provided appear to be (a) that PFI avoids a particular difficulty faced by another constraint that has been proposed in the literature--namely Rott's `Qualified Intersection' principle \cite{rott2001change}--and which can be reformulated in a manner that is superficially rather similar to PFI and (b) that PFI entails a pair of prima facie appealing principles. 
%
%
%(b) that PFI entails a pair of principles that are both prima facie appealing and ``appear to be the analogues of the first two of the four well known postulates advocated by Darwiche and Pearl\ldots for iterated belief revision'', the principles being: 
%\begin{tabbing}
%BLAHBLI: \=\kill
%\textcolor{red}{(???)}   \>  Given $B\in\bel{\Psi\contract A}$, if $A \in \mbox{Cn}(B)$ then \\
%\> $\bel{(\Psi \contract A) \contract B} = \bel{\Psi \contract A}\cap \bel{\Psi \contract B}$ \\[0.1cm]
%\textcolor{red}{(???)}   \>  Given $B\in\bel{\Psi\contract A}$, if $\neg A \in \mbox{Cn}(B)$ then \\
%\> $\bel{(\Psi \contract A) \contract B} = \bel{\Psi \contract A}\cap \bel{\Psi \contract B}$   \\[-0.25em]
%\end{tabbing} 
% \vspace{-0.5em}
Neither of these considerations strike us as being particularly compelling. For one, Rott's Qualified Intersection principle remains itself unclearly motivated. Secondly, plenty of ill-advised principles can be shown to have certain plausible consequences. 

The TeamQueue approach, however, allows us to rest the principle on a far firmer foundation. Indeed:
%
%Nayak {\em et al} also never explain the precise sense in which the pair of principles that they derive from PFI could be considered to be ``analogues'' of the first two Darwiche-Pearl postulates.
%
%~
%
%\textcolor{red}{So are the two above principles analogues of the 1st 2 DP postulates?}
%
%~
%
%\subsection{From Revision to Contraction}
%
\begin{proposition}
Let $\oplus$ be a TeamQueue combinator, let $\ast$ be an AGM revision operator and let $\contract$ be such that $\preceq_{\Psi \contract A}$ is defined from $\ast$ via COMBI using $\oplus$. If $\ast$ satisfies CR$\ast 1$ and CR$\ast 2$ then $\contract$ satisfies PFI.
\end{proposition}

\section{The Synchronous TeamQueue Combinator}

%\subsection{The Synchronous TeamQueue Combinator}

%\subsection{Synchronous TeamQueue}

A special case of TeamQueue combinators takes $a_{\preceq_1, \preceq_2}(i) = \{1,2\}$ for all ordered pairs $\langle \preceq_1, \preceq_2 \rangle$ and all $i$. This represents a particularly {\em fair} way of combining tpos. In terms of our supermarket analogy, it corresponds to the situation in which the tills process groups of customers at the same speed.
\begin{definition}
The {\em Synchronous TeamQueue} (STQ) combinator is the TeamQueue combinator for which $a_{\preceq_1, \preceq_2}(i) = \{1,2\}$ for all ordered pairs $\langle \preceq_1, \preceq_2 \rangle$ and all $i$. We will denote the STQ combinator by $\STQ$.
\end{definition}
\begin{example}
Suppose $W = \{x,y,z,w\}$, that $\preceq_1$ is the tpo represented by the ordered partition $\langle\{z\}, \{w\}, \{x,y\}\rangle$ and $\preceq_2$ is represented by $\langle\{x,z\}, \{y\}, \{w\}\rangle$. Then the ordered partition corresponding to $\preceq_{1\STQ 2}$ is $\langle T_1, T_2\rangle = \langle\{x,z\}, \{w,y\}\rangle$.
%, since 
%\begin{eqnarray*}
%T_1 & = & \bigcup_{j \in \{1,2\}} \min(W, \preceq_j) = \{x,z\} \\
%T_2 & = & \bigcup_{j \in \{1,2\}} \min(T_{1}^c, \preceq_j) = \{w,y\} \\
%\end{eqnarray*}
%
\end{example}
Roughly, $\preceq_{1\STQ 2}$ tries to make each world as low in the ordering as possible, while trying to preserve the information contained in $\preceq_1$ and $\preceq_2$. (The idea is similar to that of the {\em rational closure} construction in default reasoning \cite{lehmann1992does}.)
We remark that $\STQ$ is commutative, i.e., $\combi=\preceq_{2\oplus 1}$. 
%We have a number of characterisations of this combinator. 
%
%We can show, for instance, that it yields the `flattest' tpo that yields the lower bound property LB in the presence of COMBI. 
%
%\textcolor{red}{Note: this only proves, I think, that it is {\em a} flattest. Check to see if stronger result can be recovered}
%
%More precisely:
%\begin{proposition}
%If $\preceq_{1\oplus 2}$ can be obtained from $\preceq_{1\STQ 2}$ by lowering the rank of one or more worlds without increasing the rank of any world, and $\contract$ is derived from $\ast$ via $\oplus$ and COMBI, then there exists $C\in L$, such that $C\in \bel{\Psi \ast B} \cap \bel{(\Psi \ast \neg A) \ast B}$ but $C\notin \bel{(\Psi \contract A) \ast B}$
%\end{proposition}
It can be characterised semantically, in the absence of domain restrictions, as follows:
\begin{theorem}
$\STQ$ is the only basic combinator that satisfies both $\oplus$SPU+ and the following `Parity' constraint:
\begin{tabbing}
BLAHBLI: \=\kill
($\oplus$PAR)   \>  If $x \prec_{1 \oplus 2} y$ then for each $i \in \{1,2\}$ there exists $z$ s.t. $x \sim_{1 \oplus 2} z$ and\\
\> $z \prec_i y$\\[-0.25em]
\end{tabbing} 
\vspace{-0.5em}
%
%\vspace{0.2cm}
%\\
%\begin{tabular}{@{}ll}
%($\oplus$PAR) & If $x \prec_{1 \oplus 2} y$ then for each $i \in \{1,2\}$ there exists $z$  \\
%& s.t. $x \sim_{1 \oplus 2} z$ and $z \prec_i y$
%\end{tabular}
\end{theorem}
Note that $\oplus$WPU+ is not listed among the characteristic principles: it is entailed by the conjunction of $\oplus$SPU+ and $\oplus$PAR. 

Whilst $\oplus$PAR may not be immediately easy to grasp, it can be given a nice formulation in our setting in terms of the notion of {\em strong belief} \cite{battigalli2002strong,stalnaker1996knowledge}. A sentence $A \in \bel{\Psi}$ is strongly believed in $\Psi$ in case the only way it can be dislodged by the next revision input $B$ is if $B$ is logically {\em inconsistent} with $A$. %In other words:
%\begin{definition}
That is, 
%In a belief state $\Psi$ we say sentence 
$A$ is strongly believed in $\Psi$ iff {\em (i)} $A \in \bel{\Psi}$, and {\em (ii)} $A \in \bel{\Psi \ast B}$ for all sentences $B$ such that $A \wedge B$ is consistent.
%\end{definition}
Semantically, a consistent sentence $A$ is strongly believed in $\Psi$ iff every $A$-world is strictly more plausible than every $\neg A$-world, i.e., $x \prec_\Psi y$ for every $x \in \mods{A}$, $y \in \mods{\neg A}$. With this in hand, one can show:
\begin{proposition}
$\oplus$PAR is equivalent to: %the conjunction of the left-to-right half of $\oplus$HI ($\min(\combi, W) \subseteq \min(\preceq_1, W) \cup \min(\preceq_2, W)$) and:
%\begin{tabbing}
%BLAHBLI: \=\kill
%(SB)   \>  If $\neg B$ is strongly believed in $\Psi_{1\oplus 2}$ then\\
%\> $\bel{\Psi_{1\oplus 2} \ast B}\subseteq \bel{\Psi_1 \ast B} \cap \bel{\Psi_2 \ast B}$\\[-0.25em]
%\end{tabbing} 
\begin{tabbing}
BLAHBL: \=\kill
%(SB)   \>  If $x \prec_{\Psi_{1\oplus 2}} y$ for every $x \in \mods{\neg B}$, $y \in \mods{B}$, then \\
%\> $\min( \preceq_{1},\mods{B})\cup\min(\preceq_{2},\mods{B})\subseteq \min( \preceq_{1\oplus 2},\mods{B})$\\[-0.25em]

($\oplus$SB)   \>  If $x \prec_{1\oplus 2} y$ for every $x \in S^c$, $y \in S$, then  $\min( \preceq_{1},S)\cup\min(\preceq_{2},S)\subseteq$\\
\>$ \min( \preceq_{1\oplus 2},S)$\\[-0.25em]
\end{tabbing} 
\end{proposition}
\vspace{-0.5em}
Given COMBI, $\oplus$SB yields: If $\neg B$ is strongly believed in $\Psi \contract A$ then $\bel{(\Psi \contract A) \ast B}\subseteq \bel{\Psi \ast B} \cap \bel{(\Psi \ast \neg A) \ast B}$. Thus, although we cannot have EHI for all $A,B$, the STQ combinator {\em does} guarantee it to hold for a certain restricted class of pairs of sentences, namely those $A,B$ such that $\neg B$ is strongly believed after removing $A$.

%\section{Iterated Contraction via Synchronous TeamQueue Combination}

%\subsection{Further Postulates for Iterated Revision}
To finish this section, we turn to further behaviour for iterated contraction that can be captured thanks to the further principles satisfied by $\STQ$.

Three popular approaches to supplementing the AGM postulates for revision and the DP postulates can be found in the literature: the `natural' \cite{boutilier1996iterated}, `restrained' \cite{booth2006admissible}, and `lexicographic' \cite{nayak1994iterated} approaches. All of these have the semantic consequence that the prior tpo $\preceq_{\Psi}$ determines the posterior tpo $\preceq_{\Psi\ast A}$. %(though see \cite{chandlerboothjpl} for some reservations regarding the desirability of such a consequence). 
All three promote the lowest $A$-worlds in $\preceq_\Psi$ to become the lowest overall in $\preceq_{\Psi\ast A}$, but differ on what to do with the rest of the ordering. Natural revision leaves everything else unchanged, restrained revision preserves the strict ordering $\prec_\Psi$ while additionally making every $A$-world $x$ strictly lower than every $\neg A$-world $y$ for which $x \preceq_\Psi y$, and lexicographic revision just makes every $A$-world lower than every $\neg A$-world, while preserving the ordering within each of $\mods{A}$ and $\mods{\neg A}$. 

This raises an obvious question, namely: Which principles of iterated contraction does one recover from the natural, restrained and lexicographic revision operators, respectively, if one defines $\contract$ from $\ast$ using $\STQ$? As it turns out, both the natural and the restrained revision operator yield the very same iterated contraction operator, which has been discussed in the literature under the name of `natural contraction' \cite{nayak2007iterated}, and which sets $\min(\preceq_{\Psi}, \mods{\neg A}) \cup$ \mbox{$\min(\preceq_\Psi, W)$} to be the lowest rank in $\preceq_{\Psi \contract A}$ while leaving $\preceq_{\Psi \contract A}$ otherwise unchanged from $\preceq_\Psi$.
\begin{proposition}
Let $\ast$ be any revision operator--such as the natural or restrained revision operator--satisfying the following property:
\begin{tabbing}
BLAHBLI \= BLAB\=\kill
  \>  If $x,y\notin\min(\preceq_{\Psi},\mods{A})$ and $x\prec_{\Psi} y$, then$x\prec_{\Psi \ast A} y$\\[-0.25em]
\end{tabbing} 
\vspace{-0.5em}
Let $\contract$ be the contraction operator defined from $\ast$ via COMBI  using $\STQ$. Then $\contract$ is the natural contraction operator.
%, characterised by:
%\begin{tabbing}
%BLAHBLI: \= BLAB\=\kill
%($\contract$NAT)   \> $x\preceq_{\Psi\contract A} y$ iff\\
%\> (a) \> $x\in \min(\preceq_{\Psi}, \mods{\neg A})\cup\min(\preceq_{\Psi}, W)$, or\\
%\> (b) \> $x,y\notin \min(\preceq_{\Psi}, \mods{\neg A})\cup\min(\preceq_{\Psi}, W)$ \\
%\> \> and $x\preceq_{\Psi} y$ \\[-0.25em]
%\end{tabbing} 
%\vspace{-0.5em} 
\end{proposition}
We do not have a characterisation of the operator that is recovered from lexicographic revision in this manner, which we call the {\em STQ-lex contraction} operator. That is, STQ-lex contraction sets $\preceq_{\Psi \contract A} = \preceq_\Psi \STQ \preceq_{\Psi \ast_L \neg A}$, where $\ast_L$ is lexicographic revision. We can report, however, that it is distinct from both lexicographic and priority contraction, the other two iterated contraction operators discussed in the literature alongside natural contraction \cite{nayak2007iterated}. Roughly, lexicographic contraction works by setting the $i^{\mathrm{th}}$ rank $S_i$ of $\preceq_{\Psi \contract A}$ to be the union of the $i^{\mathrm{th}}$-lowest $A$-worlds with the $i^{\mathrm{th}}$-lowest $\neg A$-worlds.  

\begin{example}
Suppose $W = \{x,y,z,w\}$ and $\preceq_\Psi$ is the tpo represented by $\langle \{x\}, \{y\}, \{z\}, \{w\} \rangle$. Let $\mods{A} = \{x,w\}$, so that $\preceq_{\Psi \ast_L \neg A} = \langle  \{y\}, \{z\}, \{x\}, \{w\} \rangle$. Then lexicographic contraction yields $\preceq_{\Psi \contract A} = \langle \{x,y\}, \{z,w\} \rangle$ while STQ-lex contraction yields $\preceq_{\Psi \contract A} = \langle \{x,y\}, \{z\},\{w\} \rangle$.
\end{example}
\noindent
Both lexicographic and priority contraction can, however, still be recovered via the TeamQueue approach. Lexicographic contraction can be recovered from lexicographic revision by combining, not $\preceq_{\Psi}$ and $\preceq_{\Psi\ast_L \neg A}$, but rather $\preceq_{\Psi\ast_L A}$ and $\preceq_{\Psi\ast_L \neg A}$ using $\STQ$. Priority contraction can be recovered from lexicographic revision by combining $\preceq_{\Psi}$ and $\preceq_{\Psi\ast\neg A}$ using a TeamQueue combinator. However, the combinator involved is not $\STQ$ but rather the TeamQueue combinator that is most `biased' towards $\preceq_2$: the combinator for which, for all ordered pairs $\langle\preceq_1,\preceq_2\rangle$, $a_{\preceq_1,\preceq_2}(1) = \{1,2\}$, then $a_{\preceq_1,\preceq_2}(j) = \{2\}$ for all $j>1$.

%\section{Comparison with Ramachandran}
%
%\begin{itemize}
%\item Ramachandran defined 3 contraction operators: natural, priority, lexicographic. Do they fit into our scheme?
%
%\item natural does. can be obtained via STQ if $\ast$ is natural revision, restrained revision.
%
%\item Priority, lexicographic do not.
%
%\item lexicographic contraction might be expected to come from STQ + lex revision. But combining STQ + lex revision actually yields a new operator distinct from Ramachandran.
%
%\item However Ramachandran's lex contraction can be recaptured if instead of combining $\preceq$ and $\preceq^\ast_{\neg A}$ we combine $\preceq^\ast_{A}$ with $\preceq^\ast_{\neg A}$ (or tweak lex revision so as to satisfy strong vacuity)
%\end{itemize}

\section{Conclusions}

We have shown that the issue of extending the Harper identity to iterated belief change (a) is not a straightforward affair but (b) can be fruitfully approached by combining a pair of total preorders by means of TeamQueue combinator. We have also noted that one particular such combinator, the Synchronic TeamQueue combinator $\STQ$ can be put to work to derive various counterparts for contraction of the three best known iterated revision operators.

Whilst the normative appeal of the characteristic syntactic properties $\oplus$LB and $\oplus$UB of the TeamQueue family of combinators is clear enough, we do not, at this stage, have a clear enough grasp of the normative appeal of the further syntactic requirement $\oplus$SB that characterises $\STQ$. We plan to investigate this issue further in future work.

A second issue that we would like to explore is the question of whether or not it is possible to show that the Darwiche-Pearl postulates are {\em equivalent} to the ones proposed by Chopra {\em et al}, given a suitable further bridge principle taking us from iterated {\em contraction} to iterated {\em revision}. Such a task would first involve providing a compelling generalisation of the Levi Identity mentioned in Proposition \ref{EHICtoR} above. 

%\textcolor{red}{Future work:}
%\begin{itemize}
%\item \textcolor{red}{Study the alternative combination alluded to in the discussion of Ramachadran's lex contraction above.}
%
%\item \textcolor{red}{Normative appeal of STQ}
%%
%%\item How about defining iterated revision in terms of iterated contraction in a kind of generalised Levi Identity?
%\end{itemize}
%
%\textcolor{red}{Secondly, it would have been interesting to see whether or not it is possible to show that the Darwiche-Pearl postulates are {\em equivalent} to the ones proposed by Chopra {\em et al}, given a suitable further bridge principle taking us from iterated {\em contraction} to iterated {\em revision}. Such a task would first involve providing a compelling generalisation of the Levi Identity mentioned in Proposition \ref{EHICtoR} above. This is not, on the face of it, a straightforward matter. }
%
%\textcolor{red}{Finally, EHI only covers the case of {\em two} iterations of the contraction operation. Whilst this was perfectly sufficient for the purposes of the present paper, an extension of the principle to cover arbitrary iterations would be of considerable interest.}

\setcounter{proposition}{0}
\setcounter{theorem}{0}

\section*{Appendix}

\vspace{1em}

\begin{proposition}\label{EHICtoR}
EHI entails 
\begin{tabbing}
BLAHBLI: \=\kill
(EHIC)   \> $[(\Psi \contract A) \contract B] = [\Psi]\cap[\Psi * \neg B] \cap [\Psi*\neg A]\cap [(\Psi * \neg A) *\neg B]$ \\[-0.25em]
\end{tabbing} 
 \vspace{-0.5em}
and is equivalent to it in the presence of AGM$\ast$3 and the Levi Identity
\begin{tabbing}
BLAHBLI: \=\kill
(LI)   \>  $[\Psi*A]=\Cn{[\Psi\contract \neg A] \cup \{A\}}$\\[-0.25em]
\end{tabbing} 
\end{proposition}
\begin{proof}
From EHI to EHIC: By HI, which EHI entails, $[(\Psi \contract A) \contract B] = [\Psi\contract A] \cap [(\Psi \contract A) * \neg B] = [\Psi] \cap [\Psi*\neg A] \cap [(\Psi \contract A) * \neg B]$. By EHI, we have $[(\Psi \contract A) * \neg B] =[\Psi * \neg B] \cap [(\Psi * \neg A) * \neg B]$ and hence  $[(\Psi \contract A) \contract B] =  [\Psi]\cap[\Psi * \neg B] \cap [\Psi*\neg A]\cap[(\Psi * \neg A) * \neg B]$ as required.  

From EHIC to EHI: By LI, we have $[(\Psi \contract A) * \neg B] =\textrm{Cn}([(\Psi\contract A)\contract B] \cup\{\neg B\})$. By EHIC, we have $[(\Psi\contract A)\contract B]= [\Psi]\cap[\Psi * \neg B] \cap [\Psi*\neg A] \cap[(\Psi * \neg A) * \neg B]$.  So to recover EHI, we need to show that $\textrm{Cn}([\Psi]\cap[\Psi * \neg B] \cap [\Psi*\neg A] \cap[(\Psi * \neg A) * \neg B]\cup\{\neg B\})=[\Psi * \neg B] \cap[(\Psi * \neg A) * \neg B]$.
The left-to-right direction, i.e.  $\textrm{Cn}([\Psi]\cap[\Psi * \neg B] \cap [\Psi*\neg A] \cap[(\Psi * \neg A) * 
\neg B]\cup\{\neg B\})\subseteq[\Psi * \neg B] \cap[(\Psi * \neg A) * \neg B]$, is immediate. Regarding the right-to-left, assume, for some arbitrary $C$, that $C\in [\Psi * \neg B] \cap[(\Psi * \neg A) * \neg B]$. Firstly, it follows by AGM$\ast$3 and the deduction theorem that $\neg B\rightarrow C\in [\Psi]$ and $\neg B\rightarrow C\in [\Psi*\neg A]$. Secondly,  it follows by deductive closure of belief sets that $\neg B\rightarrow C\in [\Psi * \neg B] \cap[(\Psi * \neg A) * \neg B]$. Therefore  $\neg B\rightarrow C\in [\Psi]\cap[\Psi * \neg B] \cap [\Psi*\neg A] \cap[(\Psi * \neg A) * \neg B]$ and hence $C\in \textrm{Cn}([\Psi]\cap[\Psi * \neg B] \cap [\Psi*\neg A] \cap[(\Psi * \neg A) * 
\neg B]\cup\{\neg B\})$, as required. 
\end{proof}

\vspace{1em}

\begin{proposition}\label{EHICtriv}
In the presence of AGM$\ast$5,  AGM$\ast$6 and AGM$\contract$3, EHI (and more specifically, HI, alongside its left-to-right half  $\bel{(\Psi \contract A) \ast B}\subseteq\bel{\Psi \ast B} \cap \bel{(\Psi \ast \neg A) \ast B}$) entails that there does not exist a belief state $\Psi$ such that: (i) $[\Psi]=\mbox{Cn}(p\wedge q)$, (ii) $[\Psi*\neg p]=\mbox{Cn}(\neg p\wedge q)$ and (iii) $[\Psi *p\leftrightarrow\neg q]=\mbox{Cn}(p\leftrightarrow\neg q)$, where $p$ and $q$ are propositional atoms.
\end{proposition}

\begin{proof}
We first show that HI and the left-to-right half of EHI jointly entail that $[(\Psi \contract A)\contract B]\subseteq[\Psi * \neg B]$. Indeed, by HI, $[(\Psi \contract A) \contract B] = [\Psi\contract A] \cap [(\Psi \contract A) * \neg B] \subseteq  [(\Psi \contract A) * \neg B]$. By the left-to-right half of EHI, we then have  $[(\Psi \contract A) \contract B] \subseteq  [\Psi * \neg B] \cap [(\Psi * \neg A) * \neg B] \subseteq [\Psi * \neg B]$ as required.  

We now establish that, in the presence of AGM$\ast$5, AGM$\ast$6 and AGM$\contract$3, HI and the left-to-right half of EHI jointly entail the following ``vacuity'' principle:
\begin{tabbing}
BLAHBLI: \=\kill
(VAC)   \> If $A$ is consistent and $B\in[\Psi  * A]$, then $[\Psi]\cap[\Psi * A]\subseteq [\Psi *  B]$
 \\[-0.25em]
\end{tabbing} 
 \vspace{-0.5em}
Indeed, assume that $A$ is consistent and that $B\in[\Psi * A]$. Since $ A$ is consistent, so too is $[\Psi * A]$, by AGM$\ast$5, and hence $ \neg B\notin[\Psi * A]$. Since, by HI,  we have $[\Psi \contract \neg A]=[\Psi]\cap[\Psi *A]$ (with help from AGM$\ast$6), it follows that $\neg B\notin[\Psi \contract \neg A]$. Given AGM$\contract$3, we then have $[(\Psi \contract \neg A)\contract \neg B]=[\Psi \contract \neg A]$, and, by HI, $[(\Psi \contract \neg A)\contract \neg B]=[\Psi]\cap[\Psi * A]$. By the inclusion $[(\Psi \contract \neg A)\contract \neg B]\subseteq[\Psi * B]$, which we have shown above to be derivable from HI and the left-to-right half of EHI (plus AGM$\ast$6), it then follows that $[\Psi]\cap[\Psi * A]\subseteq [\Psi * B]$, as required. 

With this in place, assume VAC and, for reductio, that there exists a belief set satisfying (i) to (iii). It follows from (ii) that $p\leftrightarrow\neg q \in [\Psi * \neg p]$. Given the latter, it then follows from VAC that $[\Psi]\cap[\Psi *\neg p]\subseteq [\Psi * p\leftrightarrow\neg q]$. But by (i) and (ii), $[\Psi]\cap[\Psi *\neg p]=\mbox{Cn}(p\wedge q)\cap\mbox{Cn}(\neg p\wedge q)=\mbox{Cn}(q)$. Hence, by $[\Psi]\cap[\Psi *\neg p]\subseteq [\Psi * p\leftrightarrow\neg q]$, we have $q\in  [\Psi * p\leftrightarrow\neg q]$. But (iii) tells us that $[\Psi *p\leftrightarrow\neg q]=\mbox{Cn}(p\leftrightarrow\neg q)$. Contradiction. 
\end{proof}

 \vspace{1em}

\begin{proposition}\label{SWPUandULB}
 $\oplus$UB and $\oplus$LB are respectively equivalent to 
\begin{tabbing}
BLAHBLII: \=\kill
($\oplus$SPU+)   \> If $x \prec_1 y$ and $z \prec_2 y$ then $x \prec_{1 \oplus 2} y$ or $z \prec_{1 \oplus 2} y$  \\[-0.25em]
\end{tabbing} 
\vspace{-0.5em}
and
\begin{tabbing}
BLAHBLI:I \=\kill
($\oplus$WPU+)   \> If $x \preceq_1 y$ and $z \preceq_2 y$ then either $x \combi y$ or $z \combi y$ \\[-0.25em]
\end{tabbing} 
\end{proposition} 
\vspace{-0.5em}
\begin{proof}
%
%~In what follows, we denote by $\varphi_S$ an arbitrary sentence in $L$, such that $S=\mods{\varphi}$
%
%From LB to $\oplus$SPU+: Suppose that $x\prec_{1} y$ and $z \prec_{2} y$. From the former, we know that $\varphi_{\{x,z\}}\in \bel{\Psi_1\ast\varphi_{\{x,y,z\}}}$ and from the latter we know that $\varphi_{\{x,z\}}\in \bel{\Psi_{2}*\varphi_{\{x,y,z\}}}$. Thus, by LB, $\varphi_{\{x,z\}}\in \bel{\Psi_{1\oplus 2}\ast\varphi_{\{x,y,z\}}}$. From this, it must the case that $y\notin\min(\{x,y,z\},\preceq_{1\oplus 2})$, so either $x\prec_{1\oplus 2} y$ or $z\prec_{1\oplus 2} y$, as required.
%
From $\oplus$UB to $\oplus$SPU+: Suppose that $x\prec_{1} y$ and $z \prec_{2} y$. From the former, we know that $\min(\preceq_{1}, \{x,y,z\})\subseteq \{x,z\}$ and from the latter we know that  $\min(\preceq_{2}, \{x,y,z\})\subseteq \{x,z\}$. Thus, by $\oplus$UB, $\min(\preceq_{1\oplus 2}, \{x,y,z\})\subseteq \{x,z\}$. From this, it must the case that $y\notin\min(\preceq_{1\oplus 2}, \{x,y,z\})$, so either $x\prec_{1\oplus 2} y$ or $z\prec_{1\oplus 2} y$, as required.

%From $\oplus$SPU+ to $\oplus$LB: It suffices to show that $\min(\mods{B},\preceq_{1\oplus 2})\subseteq\min(\mods{B},$ $\preceq_{1})\cup\min(\mods{B},\preceq_{2})$. Assume for contradiction that there exists an $x$, such that $x\in \min(\mods{B},$ $\preceq_{1\oplus 2})$ but $x\notin \min(\mods{B},\preceq_{1})\cup\min(\mods{B},\preceq_{2})$. From the latter, there exist $y,z\in\mods{B}$, such that $y\prec_\Psi x$ and $z\prec_{2} x$. By $\oplus$SPU+, it then follows that either $y\prec_{1\oplus 2} x$ or $z\prec_{1\oplus 2} x$, contradicting  $x\in \min(\mods{B},\preceq_{1\oplus 2})$. Thus, $\min(\mods{B},\preceq_{1\oplus 2})\subseteq\min(\mods{B},\preceq_{1})\cup\min(\mods{B},\preceq_{2})$, as required.

From $\oplus$SPU+ to $\oplus$UB: Assume for contradiction that there exists an $x$, such that $x\in \min(\preceq_{1\oplus 2}, S)$ but $x\notin \min(\preceq_{1}, S)\cup\min(\preceq_{2}, S)$. From the latter, there exist $y,z\in S $, such that $y\prec_{1} x$ and $z\prec_{2} x$. By $\oplus$SPU+, it then follows that either $y\prec_{1\oplus 2} x$ or $z\prec_{1\oplus 2} x$, contradicting  $x\in \min(\preceq_{1\oplus 2}, S)$. Thus, $\min(\preceq_{1\oplus 2}, S)\subseteq\min(\preceq_{1}, S)\cup\min(\preceq_{2}, S)$, as required.

%From UB to $\oplus$WPU+: We derive the contrapositive of $\oplus$WPU+, namely:  
%\begin{itemize}
%
%\item[] If $y\prec_{1\oplus 2} x$ and $y\prec_{1\oplus 2} z$, then $y\prec_\Psi x$ or $y\prec_{2} z$
%
%\end{itemize}
%Assume then that $y\prec_{1\oplus 2} x$ and $y\prec_{1\oplus 2} z$.  It follows from this that $\varphi_{\{y\}}\in \bel{\Psi_{1\oplus 2}*\varphi_{\{x,y,z\}}}$. By UB, we then recover either (i) $\varphi_{\{y\}}\in \bel{\Psi_{1} \ast \varphi_{\{x,y,z\}}}$ or (ii) $\varphi_{\{y\}}\in \bel{\Psi_{2}\ast \varphi_{\{x,y,z\}}}$. Assume (i). It follows that $y\prec_1 x$. Assume (ii). It follows that $y\prec_{2} z$. Hence, either $y\prec_\Psi x$ or $y\prec_{2} z$, as required.

From $\oplus$LB to $\oplus$WPU+: We derive the contrapositive of $\oplus$WPU+, namely:  
\begin{itemize}

\item[] If $y\prec_{1\oplus 2} x$ and $y\prec_{1\oplus 2} z$, then $y\prec_1 x$ or $y\prec_{2} z$

\end{itemize}
Assume then that $y\prec_{1\oplus 2} x$ and $y\prec_{1\oplus 2} z$.  It follows from this that $\min(\preceq_{1\oplus 2}, \{x,y,z\})\subseteq \{y\}$. By $\oplus$LB, we then recover either (i) $\min(\preceq_{1}, \{x,y,z\})\subseteq \{y\}$ or (ii) $\min(\preceq_{2}, \{x,y,z\})\subseteq \{y\}$. Assume (i). It follows that $y\prec_1 x$. Assume (ii). It follows that $y\prec_{2} z$. Hence, either $y\prec_1 x$ or $y\prec_{2} z$, as required.

From $\oplus$WPU+ to $\oplus$LB: Assume for reductio that $\oplus$LB fails, so that there exist an $x$ and a $y$ such that  $y\in \min(\preceq_{1}, S)$ and $z\in\min(\preceq_{2}, S)$, but $y,z\notin\min(\preceq_{1\oplus 2}, S)$. From the latter, there exist an $x_1$ and $x_2$ such that $x_1,x_2\in S $, $x_1\prec_{1\oplus 2} y$ and $x_2\prec_{1\oplus 2} z$. Since $\preceq_{1\oplus 2}$ is a total preorder, we may assume that there exists an $x$ such that $x\in S $, $x\prec_{1\oplus 2} y$ and $x\prec_{1\oplus 2} z$. By $\oplus$WPU+, we then have either $x\prec_1 y$ or $x\prec_{2} z$, contradicting our assumption that $y\in \min(\preceq_{1}, S)$ and $z\in\min(\preceq_{2}, S)$.

\end{proof}

\vspace{1em}

\begin{theorem}\label{TQNO}
$\oplus$ is a TeamQueue combinator iff it is a basic combinator that satisfies $\oplus$SPU+, $\oplus$WPU+ and the following `no overtaking' property;
\begin{tabbing}
BLAHBLI: \=\kill
($\oplus$NO)   \> If either (i) $x \preceq_1 y$ and $z \prec_2 y$ or (ii) $x \preceq_2 y$ and $z \prec_1 y$, then either \\
\> $x \combi y$ or $z \prec_{1\oplus 2} y$ \\[-0.25em]
\end{tabbing} 
\end{theorem}
\vspace{-0.5em}
\begin{proof}
We prove that $\oplus$ satisfies $\oplus$SPU+, $\oplus$WPU+ and $\oplus$NO iff it satisfies
\begin{tabbing}
BLAHBLI: \=\kill
($\oplus$TRI)   \> $\min(\preceq_{1\oplus 2}, S)$ is equal to either $\min(\preceq_{1}, S)$, $\min(\preceq_{2}, S)$ or $\min(\preceq_{1}, S)$\\
\> $\cup\min(\preceq_{2}, S)$ \\[-0.25em]
\end{tabbing} 
\vspace{-0.5em}
The desired result then follows from Proposition \ref{tri} below.

We first show that $\oplus$SPU+, $\oplus$WPU+ and $\oplus$NO entail $\oplus$TRI. 

We know that $\min(\preceq_{1\oplus 2}, S)\subseteq \min(\preceq_{1}, S)\cup \min(\preceq_{2}, S)$ from $\oplus$SPU+. Indeed, assume that $y\in \min(\preceq_{1\oplus 2}, S)$ but, for reductio, that $y\notin \min(\preceq_{1}, S)\cup \min(\preceq_{2}, S)$. Then $\exists x,z\in S$ such that $x\prec_1 y$ and $z\prec_2 y$. Then, by $\oplus$SPU+, either $x\prec_{1\oplus 2} y$ or $z\prec_{1\oplus 2} y$. Either way, we get $y\notin\min(\preceq_{1\oplus 2}, S)$. Contradiction. Hence, $y\in \min(\preceq_{1}, S)\cup \min(\preceq_{2}, S)$, as required.

Now if the converse holds, i.e.~$\min(\preceq_{1}, S)\cup \min(\preceq_{2}, S)\subseteq\min(\preceq_{1\oplus 2}, S)$, then we have $\min(\preceq_{1\oplus 2}, S)=\min(\preceq_{1}, S)\cup \min(\preceq_{2}, S)$ and we are done. So assume $\min(\preceq_{1}, S)\cup \min(\preceq_{2}, S)\nsubseteq\min(\preceq_{1\oplus 2}, S)$. Then either $\min(\preceq_{1}, S)\nsubseteq\min(\preceq_{1\oplus 2}, S)$ or $\min(\preceq_{2}, S)\nsubseteq\min(\preceq_{1\oplus 2}, S)$. Let's assume $\min(\preceq_{1}, S)\nsubseteq\min(\preceq_{1\oplus 2}, S)$. We will show that this implies $\min(\preceq_{1\oplus 2}, S)=\min(\preceq_{2}, S)$, which will suffice. (If instead we assume $\min(\preceq_{2}, S)\nsubseteq\min(\preceq_{1\oplus 2}, S)$, then the same reasoning will show $\min(\preceq_{1\oplus 2}, S)=\min(\preceq_{1}, S)$, which also suffices.) Since $\min(\preceq_{1}, S)\nsubseteq\min(\preceq_{1\oplus 2}, S)$, let $x\in \min(\preceq_{1}, S)$ but $x\notin \min(\preceq_{1\oplus 2}, S)$.

We first derive $\min(\preceq_{1\oplus 2}, S)\subseteq\min(\preceq_{2}, S)$. Let $y\in\min(\preceq_{1\oplus 2}, S)$ and assume for reductio that $y\notin \min(\preceq_{2}, S)$. Then $\exists z\in S$ such that $z\prec_{2} y$. From $y\in\min(\preceq_{1\oplus 2}, S)$, we know that $y\preceq_{1\oplus 2} z$. From $x\in \min(\preceq_{1}, S)$, we also know that $x\preceq_{1} y$. From $z\prec_{2} y$, $y\preceq_{1\oplus 2} z$ and $x\preceq_{1} y$, we can deduce by $\oplus$NO that $x\preceq_{1\oplus 2} y$, in contradiction with $x\not\in\min(\preceq_{1\oplus 2}, S)$. Hence, $y\in \min(\preceq_{2}, S)$, as required.

We now derive $\min(\preceq_{2}, S)\subseteq \min(\preceq_{1\oplus 2}, S)$. Let $y\in \min(\preceq_{2}, S)$ and assume, for reductio, that $y\notin\min(\preceq_{1\oplus 2}, S)$. From $x,y\notin\min(\preceq_{1\oplus 2}, S)$, $\exists z\in S$, such that $z\prec_{1\oplus 2}x$ and $z\prec_{1\oplus 2}y$. Then, from $\oplus$WPU+, we have either $z\prec_{1} x$ or $z\prec_{2} y$. If $z\prec_{1} x$, then we contradict $x\in \min(\preceq_{1}, S)$. If $z\prec_{2} y$, then we contradict $y\in \min(\preceq_{2}, S)$. Either way, we get a contradiction, so $y\in\min(\preceq_{1\oplus 2}, S)$, as required.

Finally, we show that $\oplus$TRI entails $\oplus$SPU+, $\oplus$WPU+ and $\oplus$NO. 

Regarding $\oplus$SPU+: From $\oplus$TRI, we know that, $\forall S$, $\min(\preceq_{1\oplus 2}, S)\subseteq \min(\preceq_{1}, S)\cup \min(\preceq_{2}, S)$. Now suppose that $x\prec_{1} y$ and $z\prec_{2} y$. Then $y\notin\min(\preceq_{1},\{x,y,z\})\cup \min(\preceq_{2},\{x,y,z\})$. Hence $y\notin\min(\preceq_{1\oplus 2},\{x,y,z\})$, so $x\prec_{1\oplus 2} y$ or $z\prec_{1\oplus 2} y$, as required.

Regarding $\oplus$WPU+: From $\oplus$TRI, we know that, $\forall S$, either $\min(\preceq_{1}, S)\subseteq \min(\preceq_{1\oplus 2}, S)$ or $\min(\preceq_{2}, S)\subseteq \min(\preceq_{1\oplus 2}, S)$. This is the property $\oplus$LB and we already proved in Proposition \ref{SWPUandULB} that it entails $\oplus$WPU+.

Regarding $\oplus$NO: From $\oplus$TRI, we know that, $\forall S$, $i\neq j$, either $ \min(\preceq_{1\oplus 2}, S)\subseteq \min(\preceq_{i}, S)  $ or $\min(\preceq_{j}, S)\subseteq \min(\preceq_{1\oplus 2}, S)$. Now assume $x\prec_{i} y$, $y\preceq_{1\oplus 2} x$, $z\preceq_{j} y$ and, for reductio, $y \prec_{1\oplus 2} z$. From $y\preceq_{1\oplus 2} x$ and $y \prec_{1\oplus 2} z$, we get $y\in\min(\preceq_{1\oplus 2},\{x,y,z\})$ but from $x\prec_{i} y$, we get $y\notin\min(\preceq_{i},\{x,y,z\})$. Hence $\min(\preceq_{1\oplus 2},\{x,y,z\})\nsubseteq\min(\preceq_{i},\{x,y,z\})$. From this and the property cited at the beginning of this paragraph, we get $\min(\preceq_{j},\{x,y,z\})\subseteq\min(\preceq_{1\oplus 2},\{x,y,z\})$. We also know from $\oplus$TRI that $\min(\preceq_{1\oplus 2},\{x,y,z\})\subseteq \min(\preceq_{1},\{x,y,z\})\cup \min(\preceq_{2},\{x,y,z\})$. Hence, since $y\in\min(\preceq_{1\oplus 2},\{x,y,z\})$ and $y\notin\min(\preceq_{i},\{x,y,z\})$, we get $y\in\min(\preceq_{j},\{x,y,z\})$. Hence, since $z\preceq_{j} y$, $z\in\min(\preceq_{j},\{x,y,z\})$ and so, from $\min(\preceq_{j},\{x,y,z\})\subseteq\min(\preceq_{1\oplus 2},\{x,y,z\})$, $z\in\min(\preceq_{1\oplus 2},\{x,y,z\})$, contradicting $y\prec_{1\oplus 2} z$. Hence $z \preceq_{1\oplus 2} y$, as required.

\end{proof}

\vspace{1em}

\begin{proposition}\label{tri}
$\oplus$ is a TeamQueue combinator iff it is a basic combinator that satisfies the following `trifurcation' property:
\begin{tabbing}
BLAHBLI: \=\kill
($\oplus$TRI)   \> $\min(\preceq_{1\oplus 2}, S)$ is equal to either $\min(\preceq_{1}, S)$, $\min(\preceq_{2}, S)$ or $\min(\preceq_{1}, S)$\\
\> $\cup\min(\preceq_{2}, S)$ \\[-0.25em]
\end{tabbing} 
\end{proposition}
\vspace{-0.5em}
\begin{proof}
Right-to-left direction: Let $\oplus$ be any combinator that satisfies those properties. We must specify a sequence $a_{\preceq_1,\preceq_2}$ for each ordered pair $\langle\preceq_1,\preceq_2\rangle$ such that (i) $\oplus_a$ satisfies properties (a1) and (a2) and (ii) $\oplus_a=\oplus$. 

Assume that $\langle S_1, S_2,\ldots, S_n\rangle$ represents $\preceq_{1\oplus 2}$. Then we specify  $a_{\preceq_1,\preceq_2}$ by setting, for all $i$,
\[
j \in a_{\preceq_1,\preceq_2}(i)\
\textrm{iff }
\min(\bigcap_{k < i}S_k^c, \preceq_j)
\subseteq 
S_i
(= \min(\bigcap_{k < i}S_k^c, \preceq_{1\oplus 2})
\]
Regarding (i), $\oplus_a$ satisfies (a1) since $\oplus$ satisfies $\oplus$TRI and (a2) since $\oplus$ satisfies $\oplus$HI

Regarding (ii), let $\langle T_1, T_2,\ldots, T_m\rangle$ represent $\preceq_{1\oplus_a 2}$. We prove by induction that $T_i=S_i$. Regarding $i=1$: The result follows from $\oplus$HI. Regarding the inductive step: Assume $T_j=S_j$, $\forall j<i$. We want to show $T_i=S_i$. By construction, $T_i=\bigcup_{j\in a(i)}\min(\preceq_j ,\bigcap_{k<i}S_k^c)$. So we need to show $\min(\bigcap_{k < i}S_k^c, \preceq_{1\oplus 2})=\bigcup_{j\in a(i)}\min(\preceq_j ,\bigcap_{k<i}S_k^c)$. This follows from $\oplus$TRI.

Left-to-right direction: We show that $\oplus_a$ satisfies each of $\oplus$SPU+, $\oplus$WPU+ and $\oplus$NO.
\begin{itemize}

\item[-] Regarding $\oplus$SPU+: We prove the contrapositive. Suppose $y\preceq_{1\oplus 2}x$ and $y\preceq_{1\oplus 2} z$. Assume $y\in S_i=\bigcup_{j\in a(i)}\min(\preceq_j ,\bigcap_{k<i}S_k^c)\subseteq\min(\preceq_1 ,\bigcap_{k<i}S_k^c)\cup \min(\preceq_2 ,\bigcap_{k<i}S_k^c)$. Assume $y\in\min(\preceq_1 ,\bigcap_{k<i}S_k^c)$. Since $y\preceq_{1\oplus 2} x$, we know that $x\in \bigcap_{k<i}S_k^c$, hence $y\preceq_{1} x$, as required. Similarly, if $y\in\min(\preceq_2 ,\bigcap_{k<i}S_k^c)$, then $y\preceq_{2} z$.

\item[-] Regarding $\oplus$WPU+: We prove the contrapositive. Suppose $y\prec_{1\oplus 2}x$ and $y\prec_{1\oplus 2}z$. Assume $y\in S_i$. Since $y\prec_{1\oplus 2}x$ and $y\prec_{1\oplus 2}z$, we know that $x,z\in \bigcap_{k<i}S_k^c\cap S_i^c$. Now, we know that $S_i$ equals one of $\min(\preceq_1,\bigcap_{k<i}S_k^c)$,  $\min(\preceq_2,\bigcap_{k<i}S_k^c)$ or $\min(\preceq_1,\bigcap_{k<i}S_k^c)\cup\min(\preceq_2,\bigcap_{k<i}S_k^c)$. We consider each case in turn:
\begin{itemize}

\item[(1)] $S_i=\min(\preceq_1,\bigcap_{k<i}S_k^c)$: From $y\in S_i$ and $x\in \bigcap_{k<i}S_k^c\cap S_i^c$, we have $y\prec_1 x$, as required.

\item[(2)] $S_i=\min(\preceq_2,\bigcap_{k<i}S_k^c)$: From $y\in S_i$ and $z\in \bigcap_{k<i}S_k^c\cap S_i^c$, we have $y\prec_2 z$, as required.

\item[(3)] $S_i=\min(\preceq_1,\bigcap_{k<i}S_k^c)\cup\min(\preceq_2,\bigcap_{k<i}S_k^c)$: Either $y\in\min(\preceq_1,\bigcap_{k<i}S_k^c)$, in which case $y\prec_1 x$, or $y\in\min(\preceq_2,\bigcap_{k<i}S_k^c)$, in which case $y\prec_2 z$.

\end{itemize}

\item[-] Regarding $\oplus$NO: We show: If $x\prec_i y$, $y\preceq_{1\oplus 2} x$ and $z\preceq_j y$, then $z\preceq_{1\oplus 2} y$, $i\neq j$, $i,j\in\{1,2\}$. Suppose that $x\prec_i y$, $y\preceq_{1\oplus 2} x$ and $z\preceq_j y$. We must show that $z\preceq_{1\oplus 2} y$. Assume $y\in S_t$. Then, from $y\preceq_{1\oplus 2} x$ and $z\preceq_j y$, we have $x,z\in \bigcap_{k<t}S_t^c$ and furthermore $z\in S_t^c$. We know that $S_t$ equals one of $\min(\preceq_1,\bigcap_{k<t}S_k^c)$,  $\min(\preceq_2,\bigcap_{k<t}S_k^c)$ or $\min(\preceq_1,\bigcap_{k<t}S_k^c)\cup\min(\preceq_2,\bigcap_{k<t}S_k^c)$. From $x\prec_i y$, we know that $y\notin\min(\preceq_i, \bigcap_{k<t}S_k^c)$, hence we must have $y\in\min(\preceq_j, \bigcap_{k<t}S_k^c)$. Furthermore, we are left with either $S_t=\min(\preceq_j, \bigcap_{k<t}S_k^c)$ or $S_t=\min(\preceq_1,\bigcap_{k<t}S_k^c)\cup\min(\preceq_2,\bigcap_{k<t}S_k^c)$. In either case, we must have $y\prec_j z$, as required. 

\end{itemize}

\end{proof}

\vspace{1em}

\begin{proposition}\label{TQ}
Given $\oplus \textit{DOM}$, $\oplus$ is a TeamQueue combinator iff it satisfies $\oplus$SPU+ and $\oplus$WPU+.
\end{proposition}
\begin{proof}
We show that, given $\oplus \textit{DOM}$, if $\oplus$ satisfies $\oplus$SPU+ and $\oplus$WPU+, then it satisfies $\oplus$NO and hence, by Propositions \ref{TQNO} and \ref{SWPU+andSWPU}, is a TeamQueue combinator. 

Suppose $x\prec_i y$, $y\preceq_{1\oplus 2} x$ and $z\preceq_j y$, with $i\neq j$. We must show $z\preceq_{1\oplus 2} y$. If we can show $z\preceq_i y$, then we can conclude $z\preceq_{1\oplus 2} y$ from $\oplus$WPU. So suppose for reductio that $y\prec_i z$. From $\oplus$DOM, $\exists S$, such that, $\forall u,v\in S$, $u\preceq_1 v$ iff $u\preceq_2 v$ and $\forall u,v\in S^c$, $u\preceq_1 v$ iff $u\preceq_2 v$. From $z\preceq_j y$ and $y\prec_i z$, it must be the case that $y\in S$ and $z\in S^c$. If $x\in S$, then from $x\prec_i y$, we get $x\prec_j y$ and so $x\prec_{1\oplus 2} y$ from $\oplus$SPU, contradicting $y\preceq_{1 \oplus 2} x$. If $x\in S^c$, then, since $x\prec_i y\prec_i z$ and $z\in S^c$, $x\prec_j z$. So from this and $z\preceq_j y$, we get $x\prec_j y$ and so again $x\prec_{1\oplus 2} y$ from $\oplus$SPU, contradicting $y\preceq_{1 \oplus 2} x$. Hence, it must be that $z\preceq_i y$, as required. 

\end{proof}

\vspace{1em}

\begin{proposition}\label{SWPU+andSWPU}
Given $\oplus \textit{DOM}$, $\oplus$ satisfies $\oplus$SPU+ and $\oplus$WPU+ iff it satisfies $\oplus$SPU and $\oplus$WPU, respectively.
\end{proposition}
\begin{proof}
We prove this by demonstrating the equivalence, given $\oplus$DOM, of $\oplus$SPU and $\oplus$WPU with $\oplus$UB and $\oplus$LB, respectively, which we have shown (see Proposition \ref{SWPUandULB}) to be equivalent to $\oplus$SPU+ and $\oplus$WPU+, respectively. 

Regarding $\oplus$SPU and $\oplus$UB, our proof is direct. Regarding $\oplus$WPU and $\oplus$LB, we first show that  $\oplus$WPU is equivalent to the following weakening $\oplus$WLB of $\oplus$LB:
\begin{tabbing}
BLAHBLI: \=\kill
($\oplus$WLB)   \> $\min(\preceq_1, S)\cap \min(\preceq_2, S) \subseteq \min(\combi, S)$ \\[-0.25em]
\end{tabbing} 
\vspace{-0.5em}
before showing that $\oplus$WLB is equivalent to $\oplus$LB under the domain restriction $\oplus$DOM.

From $\oplus$UB to $\oplus$SPU: The result follows from the fact that $x\preceq y$ iff $\min(\preceq,\{x,y\})\subseteq \{x\}$.

%
%We prove this by demonstrating the equivalence, given $\oplus$DOM, of $\oplus$SPU and $\oplus$WPU with UB and LB, respectively, which we have shown (see Proposition \ref{SWPUandUUB}) to be equivalent to $\oplus$SPU+ and $\oplus$WPU+, respectively. 
%
%Regarding $\oplus$SPU and UB, the proof is direct. Regarding $\oplus$WPU and LB, we first show that  $\oplus$WPU is equivalent to the following weakening WLB of LB:
%\begin{tabbing}
%BLAHBLI: \=\kill
%(WLB)   \>  $[\Psi_{1\oplus 2}  \ast B]\subseteq \mbox{Cn}([\Psi_{1}\ast B]\cup[\Psi_{2} \ast B])$\\[-0.25em]
%\end{tabbing} 
%\vspace{-0.5em}
%%\begin{tabbing}
%%BLAHBLI: \=\kill
%%(WLB)   \>  $[(\Psi\contract A)  \ast B]\subseteq \mbox{Cn}([\Psi\ast B]\cup[(\Psi\ast\neg A) \ast B])$\\[-0.25em]
%%\end{tabbing} 
%%\vspace{-0.5em}
%before showing that WLB is equivalent to LB under the domain restriction $\oplus$DOM. As we did above, we denote by $\varphi_S$ an arbitrary sentence in $L$, such that $S=\mods{\varphi}$
%
%From UB to $\oplus$SPU: The result follows from the fact that $x\prec_{1} y$ iff $\varphi_{\{x\}}\in \bel{\Psi \ast \varphi_{\{x,y\}}}$.

From $\oplus$SPU to $\oplus$UB: It suffices to show that $\min(\preceq_{1\oplus 2},  S   )\subseteq\min(\preceq_{1},  S  )\cup\min(\preceq_{2},  S   )$. Assume $\oplus$DOM, $\oplus$SPU and that there exists an $x$, such that $x\in \min(\preceq_{1\oplus 2},  S   )$ but, for contradiction, that $x\notin \min(\preceq_{1},  S  )\cup\min(\preceq_{2},  S   )$. From the latter, there exist $y_1,y_2\in S $, such that (i) $y_1 \prec_{1}  x$ and (ii) $y_2\prec_{2}   x$. From the former, (iii) $x\preceq_{1\oplus 2}   y_1$ and (iv) $x\preceq_{1\oplus 2}   y_2$. From (i) and (iii) on the one hand and (ii) and (iv) on the other, by $\oplus$SPU, we recover (v) $x\preceq_{2}   y_1$ and (vi) $x\preceq_{1}  y_2$, respectively. The conjunctions of (i) and (vi), i.e.~$y_1 \prec_{1}  x \preceq_{1}  y_2$, and of (ii) and (v), i.e.~$y_2\prec_{2}   x\preceq_{2}   y_1$, however, jointly contradict $\oplus$DOM, since the latter entails that there exist no $x, y_1, y_2$ such that $y_1\prec_{1}  x\preceq_{1}  y_2$ but $y_2\prec_{2}   x\preceq_{2}   y_1$. Hence $x\in \min(\preceq_{1},  S  )\cup\min(\preceq_{2},  S   )$, as required.

From $\oplus$WPU to $\oplus$WLB: Let $x\in \min(\preceq_1, S)\cap \min(\preceq_2, S)$ and assume for reductio that $x\notin \min(\preceq_{1\oplus 2} ,  S  )$. Then there exists $y\in S $ such that $y\prec_{1\oplus 2}   x$. By $\oplus$WPU, either $y\prec_{1}  x$ or $y\prec_{2}   x$. Assume  $y\prec_{1}  x$ (the other case is analogous). 
Then $x\notin \min(\preceq_1, S)$ and hence $x\notin \min(\preceq_1, S)\cap \min(\preceq_2, S)$. Contradiction. Hence, $x\in \min(\preceq_{1\oplus 2} ,  S  )$, as required.

From $\oplus$WLB to $\oplus$WPU: Suppose $x\preceq_{1}  y$ and $x\preceq_{2}   y$. Then $x\in \min(\preceq_{1}, \{x,y\} )\cap \min(\preceq_{2} , \{x,y\} )$. Assume for reductio that $y\prec_{1\oplus 2}   x$. Then $x\notin\min(\preceq_{1\oplus 2}, \{x,y\})$, so, from $\oplus$WLB, $x\notin \min(\preceq_{1}, \{x,y\})\cap \min(\preceq_{2}, \{x,y\})$. Contradiction. Hence $x\preceq_{1\oplus 2}   y$, as required.

From $\oplus$LB to $\oplus$WLB: Obvious.

From $\oplus$WLB to $\oplus$LB: Assume that $\oplus$LB doesn't hold. Then there exists an $S$ such that $\min(\preceq_{1},  S  )\nsubseteq \min(\preceq_{1\oplus 2},  S   )$ and $\min(\preceq_{2},  S   )\nsubseteq \min(\preceq_{1\oplus 2},  S   )$. So there exist $x,y\in  S $ such that $x\in\min(\preceq_{1},  S  )$, $y\in\min(\preceq_{2},  S   )$ and $x, y \notin\min(\preceq_{1\oplus 2},  S   )$. Hence there exists $z\in S$ such that $z\prec_{1\oplus 2}   x$ and $z\prec_{1\oplus 2}   y$. By $\oplus$WLB, we know from $z\prec_{1\oplus 2}   x$ that either  $z\prec_{1}  x$ or  $z\prec_{2}   x$. From this and the fact that $x\in\min(\preceq_{1},  S  )$, we recover the result that $z\prec_{2}   x$. Similarly, we also recover $z\prec_{1}  y$. So we obtain the following pattern: $x\preceq_{1}  z\prec_{1}  y$ and $y\preceq_{2}   z\prec_{2}   x$. But this is not possible given $\oplus$DOM. Hence $\oplus$LB holds, as required.
\end{proof}

\vspace{1em}
\begin{proposition}\label{DPtoC}
Let $\oplus$ be a TeamQueue combinator, let $\ast$ be an AGM revision operator and let $\contract$ be such that $\preceq_{\Psi \contract A}$ is defined from $\ast$ using $\oplus$. Then, for each $i = 1,2,3,4$, if $\ast$ satisfies (CR$\ast i$) then $\contract$ satisfies (CR$\contract i$).
\end{proposition}
\begin{proof} 
From CR$\ast 1$ to CR$\contract 1$: Let $x,y\in\mods{\neg A}$. We must show that $x\preceq_{\Psi\contract A} y$ iff $x\preceq_{\Psi} y$. Note that from CR$\ast 1$, we have  (1) $x\preceq_{\Psi*\neg A} y$ iff $x\preceq_{\Psi} y$. Regarding the left-to-right direction of the equivalence: Assume (2) $y\prec x$. From (1) and (2), we recover (3) $y\prec_{\Psi*\neg A} x$. From (2) and (3), by $\oplus$SPU, it follows that $y\prec_{\Psi\contract A} x$, as required. Regarding the right-to-left-direction: Assume (4) $x\preceq_{\Psi} y$. From (1) and (4), we recover (5) $x\preceq_{\Psi*\neg A} y$. From (4) and (5), by $\oplus$WPU, it follows that $x\preceq_{\Psi\contract A} y$, as required. 

From CR$\ast 2$ to CR$\contract 2$: Similar proof as the one given for the derivation of CR$\contract 1$ from CR$\ast 1$.

From CR$\ast 3$ to CR$\contract 3$: Let $x\in\mods{ \neg A}$, $y\in\mods{A}$ and (1) $x\prec_{\Psi} y$. We must show that $x\prec_{\Psi\contract A} y$. From CR$\ast 3$, we recover (2) $x\prec_{\Psi * \neg A} y$. From (1) and (2), by $\oplus$SPU, we then obtain $x\prec_{\Psi\contract A} y$, as required.

From CR$\ast 4$ to CR$\contract 4$: Let $x\in\mods{\neg A}$, $y\in\mods{A}$ and (1) $x\preceq_{\Psi} y$. We must show that $x\preceq_{\Psi\contract A} y$. From CR$\ast 4$, we recover (2) $x\preceq_{\Psi * \neg A} y$. From (1) and (2), by $\oplus$WPU, we then obtain $x\preceq_{\Psi\contract A} y$, as required.
\end{proof}

\vspace{1em}

\begin{proposition}
Let $\ast$ be any revision operator satisfying C$\ast$1 and C$\ast$2 and $\contract$ be the contraction operator defined from * using any tpo aggregation function satisfying $\oplus$WPU, $\oplus$SPU and $\oplus$HI. Then $\contract$ satisfies PFI.
\end{proposition}
\begin{proof}
Assume that $\ast$ satisfies CR$\ast 1$ and CR$\ast 2$ and let  $\contract$ be the contraction operator defined from $\ast$ using some tpo aggregation function satisfying $\oplus$WPU, $\oplus$SPU and $\oplus$HI. We saw above, in Proposition \ref{DPtoC} that $\contract$  will also satisfy CR$\contract 1$ and CR$\contract 2$. The desired result then immediately follows from the theorem established by Ramachandran {\em et al} (2011, Theorem 1), according to which every contraction function $\contract$ obtained from a revision function $\ast$, such that $\contract$  and $\ast$ satisfy HI, satisfies PFI if it also satisfies CR$\contract 1$ and CR$\contract 2$.
\end{proof}

\vspace{1em}

\begin{theorem}
$\STQ$ is the only basic combinator that satisfies both $\oplus$SPU+ and the following `Parity' constraint:
%\vspace{0.2cm}
%\\
%\begin{tabular}{@{}ll}
%($\oplus$PAR) & If $x \prec_{1 \oplus 2} y$ then for each $i \in \{1,2\}$ there exists $z$  \\
%& s.t. $x \sim_{1 \oplus 2} z$ and $z \prec_i y$
%\end{tabular}
\begin{tabbing}
BLAHBLI: \=\kill
($\oplus$PAR)   \>  If $x \prec_{1 \oplus 2} y$ then for each $i \in \{1,2\}$ there exists $z$ s.t. $x \sim_{1 \oplus 2} z$ and \\
\>$z \prec_i y$\\[-0.25em]
\end{tabbing} 
\end{theorem}
\vspace{-0.5em}
\begin{proof}
We need to show that if $\oplus$ satisfies $\oplus$SPU+ and $\oplus$PAR, for any $\preceq_1, \preceq_2$, we have $\preceq_{1\oplus 2}=\preceq_{1\STQ 2}$. Assume that $\preceq_{1\oplus 2}=\{S_1, S_2,\ldots, S_m\}$ and $\preceq_{1\STQ 2}=\{T_1, T_2,\ldots, T_n\}$, where $S_i$, $T_i$ are the ranks of the relevant tpos, with lower ranks being the most preferred.

We will prove, by induction on $i$, that $S_i=T_i$, $\forall i$. Assume $S_j=T_j$, $\forall j < i$. We must show $S_i=T_i$.

Regarding $S_i\subseteq T_i$: Let $x\in S_i$, so that $x\preceq_{1\oplus 2}y$, $\forall y\in \bigcap_{j<i} S^{\mathsf{c}}_j$. Assume for reductio that $x\notin T_i$. Since $x\in S_i$, we know that $x\in\bigcap_{j<i} S^{\mathsf{c}}_j=\bigcap_{j<i} T^{\mathsf{c}}_j$. Hence, since $x\notin T_i$ and, by construction of $\preceq_{1\STQ 2}$, there exists $y_1\in \bigcap_{j<i} T^{\mathsf{c}}_j$ such that $y_1\prec_{1} x$ and there exists $y_2\in \bigcap_{j<i} T^{\mathsf{c}}_j$ such that $y_2\prec_{2} x$. Then, by $\oplus$SPU+, either $y_1\prec_{1\STQ 2} x$ or $y_2\prec_{1\STQ 2} x$, in both cases contradicting  $x\preceq_{1\oplus 2}y$, $\forall y\in \bigcap_{j<i} S^{\mathsf{c}}_j$.Hence $x\in T_i$, as required.

Regarding $T_i\subseteq S_i$: Let $x\in T_i$. Then, by construction of $\preceq_{1\STQ 2}$, we have $x\in\min(\preceq_1, \bigcap_{j<i} T^{\mathsf{c}}_j)\cup \min(\preceq_2, \bigcap_{j<i} T^{\mathsf{c}}_j)$. Assume for reductio that $x\notin S_i$. We know that $x\in \bigcap_{j<i} T^{\mathsf{c}}_j$, so by the inductive hypothesis, $x\in \bigcap_{j<i} S^{\mathsf{c}}_j$. From this and $x\notin S_i$ we know that there exists a $y\in S_i$, such that $y\prec_{1\oplus 2}x$. Then from $\oplus$PAR, there exist a $z_1\in S_i$ such that $z_1\prec_1 x$ and a $z_2\in S_i$ such that $z_2\prec_2 x$. But this contradicts $x\in\min(\preceq_1, \bigcap_{j<i} T^{\mathsf{c}}_j)\cup \min(\preceq_2, \bigcap_{j<i} T^{\mathsf{c}}_j)$. Hence $x\in S_i$, as required.

\end{proof}

\vspace{1em}

\begin{proposition}
$\oplus$PAR is equivalent to: %the conjunction of the left-to-right half of $\oplus$HI ($\min(\combi, W) \subseteq \min(\preceq_1, W) \cup \min(\preceq_2, W)$) and:
\begin{tabbing}
BLAHBLI: \=\kill
($\oplus$SB)   \>  If $x \prec_{1\oplus 2} y$ for every $x \in S^c$, $y \in S$, then $\min( \preceq_{1},S)\cup\min(\preceq_{2},S)\subseteq $\\
\>$\min( \preceq_{1\oplus 2},S)$\\[-0.25em]
\end{tabbing} 
\end{proposition}
\vspace{-0.5em}
\begin{proof} %From $\oplus$PAR to the left-to-right direction of $\oplus$HI: It suffices to show that  $\min( \preceq_{1},W)\cup\min(\preceq_{2},W)\subseteq \min( \preceq_{1\oplus 2},W)$. So assume $x\in \min( \preceq_{1},W)\cup\min(\preceq_{2},W)$ but, for contradiction, $x\notin  \min(\preceq_{1\oplus 2}, W)$. From the latter, by $\oplus$PAR, we know that $z_1\prec_{1}x$ for some $z_1$ such that $y\preceq_{1\oplus 2} z_1$ and $z_2\prec_{2}x$ for some $z_2$ such that $y\preceq_{1\oplus 2} z_2$. Hence $x\notin \min( \preceq_{1},W)$ and $x\notin \min(\preceq_{2},W)$> Contradiction. 
From $\oplus$PAR to $\oplus$SB: Assume that $x \prec_{1\oplus 2} y$ for every $x \in S^c$, $y \in S$. It suffices to show that $\min( \preceq_{1}, S )\cup\min(\preceq_{2}, S )\subseteq \min( \preceq_{1\oplus 2}, S )$. So assume $x\in \min( \preceq_{1}, S )\cup\min(\preceq_{2}, S )$ but, for contradiction, $x\notin  \min(\preceq_{1\oplus 2},  S )$. Then $y\prec_{1\oplus 2} x$ for some $y\in S$. From the latter, by $\oplus$PAR, we know that $z_1\prec_{1}x$ for some $z_1$ such that $y\preceq_{1\oplus 2} z_1$ and $z_2\prec_{2}x$ for some $z_2$ such that $y\preceq_{1\oplus 2} z_2$. Given our initial assumption, we can deduce from $y\preceq_{1\oplus 2} z_1$, $y\preceq_{1\oplus 2} z_2$ and $y\in S$ that $z_1, z_2\in  S $. But this, together with $z_1\prec_{1}x$ and $z_2\prec_{2}x$ contradicts $x\in \min( \preceq_{1},  S )\cup\min(\preceq_{2},  S )$. Hence $x\in  \min(\preceq_{1\oplus 2},  S )$, as required.

From $\oplus$SB %and the left-to-right direction of $\oplus$HI 
to $\oplus$PAR: Suppose $\oplus$PAR does not hold, i.e.~$\exists x,y$, such that $x\prec_{1\oplus 2}y$ and for no $z$ do we have $x\sim_{1\oplus 2}z$ and $z\prec_{1}y$ (similar reasoning will apply if we replace $\prec_{1}$ by $\prec_{2}$ here). We will show that $\oplus$SB fails, i.e.~ that $\exists S\subseteq W$, such that $x \prec_{\Psi_{1\oplus 2}} y$ for every $x \in S^c$, $y \in S$ and $\min(\preceq_{1},  S )\cup \min(\preceq_{2},  S )\nsubseteq\min(\preceq_{1\oplus 2},  S )$.

%Let $S$ be such that $\mods{B}=\{w\mid x\preceq_{1\oplus 2} w\}$ (so $\mods{\neg B}=\{w\mid w\prec_{1\oplus 2} x\}$). Since we have assumed that $x\prec_{1\oplus 2}y$, it follows that $\mods{B}\neq W$:  $\neg B$ is not a contradiction. Therefore, since $w_1\prec_{1\oplus 2}w_2$ for all $w_1\in\mods{\neg B}$ and $w_2\in\mods{B}$, it follows that $\neg B$ is strongly believed in $\Psi_{1\oplus 2}$. Clearly $x\in \mods{B}$ and, from $x\prec_{1\oplus 2} y$, we know that $y\in\mods{B}$ but $y\notin\min(\preceq_{1\oplus 2},\mods{B})$. Hence, to show $\min(\preceq_{1}, \mods{B})\cup \min(\preceq_{2}, \mods{B})\nsubseteq\min(\preceq_{1\oplus 2}, \mods{B})$ and therefore that SB fails, it suffices to show $y\in\min(\preceq_{1},\mods{B})$. But if $y\notin\min(\preceq_{1},\mods{B})$, then $z\prec_{1}y$ for some $z\in\mods{B}$, i.e. some $z$, such that $x\preceq_{1\oplus 2} z$. This contradicts our initial assumption that for no $z$ do we have $x\preceq_{1\oplus 2}z$ and $z\prec_{1}y$ . Hence $y\in\min(\preceq_{1},\mods{B})$, as required.

Let $S=\{w\mid x\preceq_{1\oplus 2} w\}$ (so that $S^c=\{w\mid w\prec_{1\oplus 2} x\}$). %Since we have assumed that $x\prec_{1\oplus 2}y$, it follows that $S\neq W$. Therefore, since $w_1\prec_{1\oplus 2}w_2$ for all $w_1\in S^c $ and $w_2\in S $, it follows that $x \prec_{\Psi_{1\oplus 2}} y$ for every $x \in S^c$, $y \in S$. 
Clearly $x\in  S $ and, from $x\prec_{1\oplus 2} y$, we know that $y\in S $ but $y\notin\min(\preceq_{1\oplus 2}, S )$. Hence, to show $\min(\preceq_{1},  S )\cup \min(\preceq_{2},  S )\nsubseteq\min(\preceq_{1\oplus 2},  S )$ and therefore that $\oplus$SB fails, it suffices to show $y\in\min(\preceq_{1}, S )$. But if $y\notin\min(\preceq_{1}, S )$, then $z\prec_{1}y$ for some $z\in S $, i.e. some $z$, such that $x\preceq_{1\oplus 2} z$. Since $\preceq_{1 \oplus 2}$ is a tpo we may assume $x\sim_{1\oplus 2} z$. This contradicts our initial assumption that for no $z$ do we have $x\sim_{1\oplus 2}z$ and $z\prec_{1}y$ . Hence $y\in\min(\preceq_{1}, S )$, as required.
\end{proof}

\vspace{1em}

\begin{proposition}
Let $\ast$ be any revision operator--such as the natural or restrained revision operator--satisfying the following property:
\begin{tabbing}
BLAHBLI \= BLAB\=\kill
  \>  If $x,y\notin\min(\preceq_{\Psi},\mods{A})$ and $x\prec_{\Psi} y$, then $x\prec_{\Psi \ast A} y$\\[-0.25em]
\end{tabbing} 
\vspace{-0.5em}
Let $\contract$ be the contraction operator defined from $\ast$ using $\STQ$. Then $\contract$ is the natural contraction operator.
\end{proposition}
\begin{proof}
Recall the definition of natural contraction:
\begin{tabbing}
BLAHBLI: \= BLAB\=\kill
($\contract$NAT)   \> $x\preceq_{\Psi\contract A} y$ iff\\
\> (a) \> $x\in \min(\preceq_{\Psi}, \mods{\neg A})\cup\min(\preceq_{\Psi}, W)$, or\\
\> (b) \> $x,y\notin \min(\preceq_{\Psi}, \mods{\neg A})\cup\min(\preceq_{\Psi}, W)$ and $x\preceq_{\Psi} y$ \\[-0.25em]
\end{tabbing} 
\vspace{-0.5em} 
We must show that for any $x, y\in W$ and $A\in L$, $x\preceq_{\Psi\contract A} y$ iff $x\preceq_{\Psi\contract_N A} y$. We split into two cases.

Case 1: $x\in \min(\preceq_{\Psi}, \mods{\neg A})\cup\min(\preceq_{\Psi}, W)$. Then, by the definitions of $\contract_N$ and $\contract$, we have both $x\preceq_{\Psi\contract A} y$ and $x\preceq_{\Psi\contract_N A} y$, so the desired result holds.

Case 2: $x\notin \min(\preceq_{\Psi}, \mods{\neg A})\cup\min(\preceq_{\Psi}, W)$. Then by definition of $\contract_N$, $x\preceq_{\Psi\contract_N A} y$ iff both $y\notin \min(\preceq_{\Psi}, \mods{\neg A})\cup\min(\preceq_{\Psi}, W)$ and $x\preceq_{\Psi} y$. We now consider each direction of the equivalence to be demonstrated separately.
\begin{itemize}
 
\item[-]  From $x\preceq_{\Psi\contract_N A} y$ to $x\preceq_{\Psi\contract A} y$: Suppose $x\preceq_{\Psi\contract_N A} y$, and hence that both $y\notin \min(\preceq_{\Psi}, \mods{\neg A})\cup\min(\preceq_{\Psi}, W)$ and $x\preceq_{\Psi} y$. Assume for reductio that $y\prec_{\Psi\contract A} x$. By $\oplus$PAR: if $y\prec_{\Psi\contract A} x$, then there exists $z$ such that $z\sim_{\Psi\contract A} y$ and $z\prec_{\Psi} x$. Hence there exists $z$ such that $z\sim_{\Psi\contract A} y$ and $z\prec_{\Psi} x$. Since $x\preceq_{\Psi} y$, we therefore also have $z\prec_{\Psi} y$. If $z\notin\min(\preceq_{\Psi}, \mods{\neg A})$, then from the postulate mentioned in the proposition, we get $z\prec_{\Psi\ast\neg A} y$ and then $z\prec_{\Psi\contract A} y$. Contradiction. Hence we can assume $z\in\min(\preceq_{\Psi}, \mods{\neg A})$. From $x\preceq_{\Psi} y$, $y\prec_{\Psi\contract A} x$ and $\oplus$WPU, we know that $y\prec_{\Psi\ast\neg A}x$. From this, CR$\ast$2,  CR$\ast$4 and $x\preceq_{\Psi} y$, we get $y\in\mods{\neg A}$. Hence, from $z\prec_{\Psi}y$ and CR$\ast$1, we recover $z\prec_{\Psi\ast\neg A} y$ and then $z\prec_{\Psi\contract A} y$ by $\oplus$SPU. Contradiction again. Hence $x\preceq_{\Psi\contract A} y$, as required.

\item[-]  From $x\preceq_{\Psi\contract A} y$ to $x\preceq_{\Psi\contract_N A} y$: Assume that $x\preceq_{\Psi\contract A} y$ and, for reductio, that either  $y\prec_{\Psi} x$ or $y\in \min(\preceq_{\Psi}, \mods{\neg A})\cup\min(\preceq_{\Psi}, W)$. If the latter holds, then we know that $y\in\min(\preceq_{\Psi\contract A}, W)$, by definition of $\contract$. Hence, from this and $x\preceq_{\Psi\contract A} y$, we also deduce that $x\in \min(\preceq_{\Psi}, \mods{\neg A})\cup\min(\preceq_{\Psi}, W)$, contradicting the assumption that $x\notin \min(\preceq_{\Psi}, \mods{\neg A})\cup\min(\preceq_{\Psi}, W)$. So assume that $y\notin \min(\preceq_{\Psi}, \mods{\neg A})\cup\min(\preceq_{\Psi}, W)$ and $y\prec_{\Psi} x$. From the latter and our assumption that $x\preceq_{\Psi\contract A} y$, it follows by $\oplus$SPU that $x\preceq_{\Psi\neg  A} y$. But it also follows from $y\notin \min(\preceq_{\Psi}, \mods{\neg A})\cup\min(\preceq_{\Psi}, W)$ and $y\prec_{\Psi} x$ that $x, y\notin \min(\preceq_{\Psi}, \mods{\neg A})$. We then recover, from the property mentioned in the proposition, the result that $x\preceq_{\Psi} y$, contradicting our assumption that $y\prec_{\Psi} x$. Hence, $x\preceq_{\Psi\contract_N A} y$, as required.

\end{itemize}

\end{proof}

\vspace{1em}

\bibliographystyle{named}
\bibliography{EHIarXiv}

\end{document}